\documentclass{article}

 \usepackage[preprint]{neurips_2026}

\usepackage[utf8]{inputenc} 
\usepackage[T1]{fontenc}    
\usepackage{hyperref}       
\usepackage{url}            
\usepackage{booktabs}       
\usepackage{amsfonts}       
\usepackage{nicefrac}       
\usepackage{microtype}      
\usepackage{xcolor}         
\usepackage{amsmath,amssymb,amsthm}
\usepackage{mathtools}
\usepackage{enumitem}
\usepackage{xcolor}
\usepackage{algorithm}
\usepackage{algpseudocode}
\usepackage{algorithmicx}
\usepackage{thmtools, thm-restate}
\usepackage{bbm}
\usepackage{tikz}
\usepackage{wrapfig}
\usepackage{caption}

\usepackage{bbm}
\usepackage{boxedminipage}
\usepackage{mdframed}
\usepackage{xcolor}
\usepackage{tcolorbox}
\usepackage{array}
\usepackage{wrapfig}
\usepackage{multirow}

\newcolumntype{C}[1]{>{\centering\arraybackslash}p{#1}}

\tcbuselibrary{theorems, skins, breakable}

\definecolor{sinkblue}{RGB}{110, 150, 200}
\definecolor{examplebg}{RGB}{245, 255, 245}

\newtcolorbox{examplebox}{
  colback=white, colframe=black, boxrule=0.5pt, arc=2pt
}

\newtcolorbox{propbox}{
  colback=examplebg, colframe=sinkblue!60!black, boxrule=0.5pt, arc=2pt,
  fonttitle=\bfseries,
  breakable, enhanced
}

\hypersetup{colorlinks=true, linkcolor=blue!60!black, citecolor=blue!60!black, urlcolor=blue!60!black}

\newtheorem{proposition}{Proposition}
\newtheorem{remark}{Remark}
\theoremstyle{definition}
\newtheorem{definition}{Definition}
\newtheorem{example}{Example}

\DeclareMathOperator{\rank}{rank}

\DeclareMathOperator{\CH}{CH}
\DeclareMathOperator{\PA}{PA}
\DeclareMathOperator{\Var}{Var}
\DeclareMathOperator{\Cov}{Cov}

\DeclareMathOperator{\NB}{NB}
\newcommand{\myyes}{\includegraphics[width=0.02\textwidth]{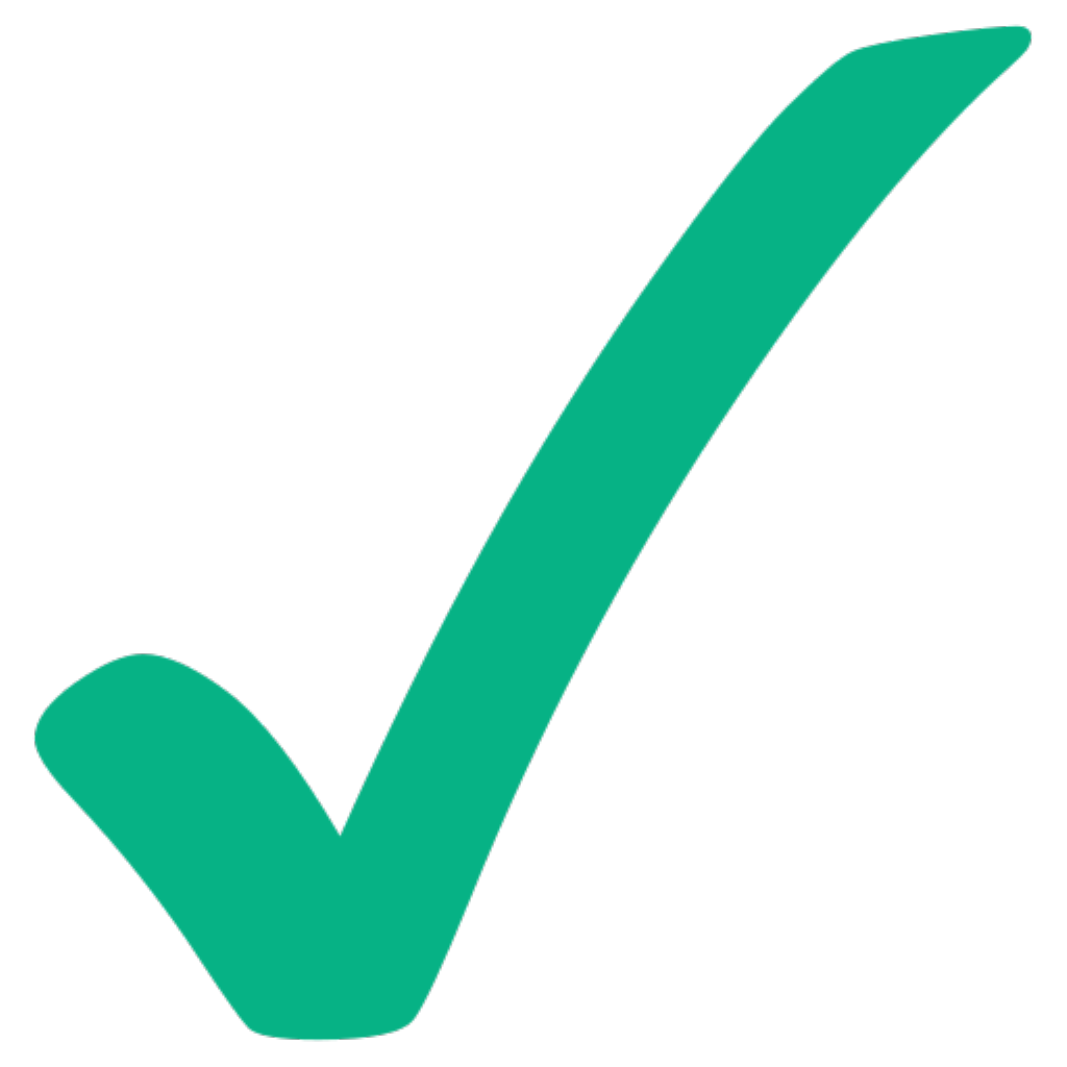}}
\newcommand{\myno}{\includegraphics[width=0.02\textwidth]{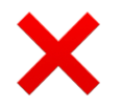}}

\newcommand{\X}{\mathbf{X}}
\newcommand{\bbR}{\mathbb{R}}
\newcommand{\bbE}{\mathbb{E}}
\newcommand{\cB}{\mathcal{B}}
\newcommand{\cD}{\mathcal{D}}
\newcommand{\cG}{\mathcal{G}}
\newcommand{\bZ}{\mathbf{Z}}
\newcommand{\ind}{\perp\!\!\!\perp}

\newcommand{\circled}[1]{\textcircled{\scriptsize #1}}
\newcommand{\var}[1]{{\scriptsize \textcolor{darkgray}{($\pm$#1)}}}

\definecolor{revisionblue}{RGB}{0, 150, 250}
\newcommand{\rev}[1]{\textcolor{black}{#1}}

\title{Causal Discovery Under Hard Selection Bias: \\ A New Robust Score-Matching Approach }

\author{%
Yiwen Qiu \\
\And
Francesco Montagna \\
\And
Shimeng Huang \\
\And 
Francesco Locatello \\
\AND \vspace{-1em}\\ 
  IST Austria\\
  Am Campus 1, Klosterneuburg, 3400\\
  \texttt{firstname.lastname@ista.ac.at} \\
}

\begin{document}
\maketitle




\newcommand{\FL}[1]{\textcolor{purple}{(FL: #1)}}

\begin{abstract}
Selection bias is a difficult yet widespread problem in causal discovery, occurring whenever non-random selection processes lead to data that is not representative of the underlying populations. 
Due to its practical importance, related works have been proposed to address the problem under {local~\citep{versteegLocalConstraintBasedCausal2022a} or interventional settings~\citep{daiWhenSelectionMeets2025}, but general algorithms for observational data remain elusive}. 
In this paper, we focus on truncated data (hard selection), and provide a general result for Additive Noise Models (ANM). We show that score-based methods are a surprisingly suitable candidate due to a special property of the score function ($\nabla \log p(x)$): \emph{deterministic truncation preserves both the score} of the joint distribution as well as the Jacobian of the score. While established score-matching algorithms still fail in most settings, we actively leverage this observation to propose a new method for causal discovery that is robust under truncated data with ANMs. Overall, our approach offers a general, robust solution that is agnostic to external information about the selection process and still achieves comparable performance to the state-of-the-art approaches when selection is not present.

\end{abstract}

\section{Introduction}
Identifying causal relationships from observational data is a fundamental problem in many scientific disciplines. However, selection bias can significantly distort the observed data distribution, leading to incorrect inference about causal relationships~\citep{heckmanSampleSelectionBias1979,hernan2004structural}. Selection bias occurs when the data collected is not representative of the underlying population due to a non-random selection process, e.g., in epidemiological studies where only patients with certain traits (socio-economic or health factors) are recorded~\citep{swansonUKBiobankSelection2012, vanaltenReweightingUKBiobank2024}, or in molecular biology where only transcripts in cells that are highly expressed are measured and included in the study~\citep{dwyer2023problem}. 

Some methods have been developed to address this problem in causal discovery, but they often rely on access to interventional data~\citep{daiWhenSelectionMeets2025, luo2025gene}, or only focus on a local structure (with three or four variables)~\citep{gregoryCausalDiscoveryData1995,versteegLocalConstraintBasedCausal2022a}. The seminal work FCI~\citep{spirtes2000causation} and its variants~\citep{colombo2012learning}  recover the causal graph under latent variables and selection bias up to a Markov equivalence class without further interpretation. 
Recent studies based on Functional Causal Models (FCMs) by ~\citet{zhangIdentifiabilityEstimationFunctional} have shown that if we know the location of the selection variable, i.e., if it is only outcome-dependent, then we can identify the full causal graph under certain assumptions. Beyond these specific, partial solutions, a general approach for selection bias in causal discovery for observational data is missing.

In this paper, we investigate the case of \textit{hard selection bias}: the densities of (some or all) observed variables are \textit{truncated}, meaning that under selection, we do not observe data on certain parts of the (true) support at all. We begin by showing a surprising yet simple result:   regardless of the location of selection, \emph{causal asymmetry persists under hard selection bias within the score function of the data ($\nabla \log p(x)$)}. This introduces an opportunity, as the score function has been a central component in a rich line of work on causal discovery~\citep{rollandScoreMatchingEnables2022,montagnaScalableCausalDiscovery2023a,montagnaScoreMatchingRoof2025}. Unfortunately, we find that while the SCORE~\citep{rollandScoreMatchingEnables2022} algorithm for non-linear additive Gaussian noise is robust to hard selection bias, the more general extensions~\citep{montagnaScoreMatchingRoof2025} are not.
To this end, we propose a novel score-based method for causal discovery, and to our knowledge, it is the first general causal discovery method that is robust to hard selection, without relying on any external information about the selection process. Our contributions are as follows:

\paragraph{Contribution} We evaluate the existing score function method---SCORE~\citep{rollandScoreMatchingEnables2022}, DAS~\citep{montagnaScalableCausalDiscovery2023a}, NoGAM~\citep{montagnaScoreMatchingRoof2025}--- and show that they can be surprisingly robust to hard selection bias (Section~\ref{sec:score}). However, failures persist: to tackle these cases, we propose a new method (Section~\ref{sec:method}) based on the conditional rank of the Hessian log density ($\nabla^2 \log(x)$), which is fully robust to hard selection bias under all Additive Noise Model (ANM) settings, with the only exception being the linear Gaussian case (see Table~\ref{tab:identifiability} for a summary). We provide a practical estimation algorithm for our method in Section~\ref{sec:estim}, and demonstrate its effectiveness through extensive synthetic and real-world \citep{pinna2010knockouts} experiments in Section~\ref{sec:exp}.

\begin{table}[h]
\centering
\renewcommand{\arraystretch}{1.2}

\begin{tabular}{lC{1.2cm}C{1.5cm}C{1.2cm}C{1.2cm}C{1.2cm}C{4cm}}
\toprule
\textbf{ANM Assumption} & {LiNGAM} & {CAM-UV} & {SCORE} & {NoGAM}  & \mbox{CHess (Ours)} \\
\midrule
Nonlinear Gaussian     & \myno & \myyes & \myyes & \myyes & \myyes \\
Linear non-Gaussian    & \myyes & \myno & \myno & \myyes & \myyes \\
Nonlinear non-Gaussian  & \myno & \myyes & \myno & \myyes & \myyes \\
\midrule
Hard Selection (Truncation) & \myno & \myno & \myyes$^{*}$ & \myno & \myyes \\
\bottomrule
\end{tabular}
\vspace{5pt}
\caption{Identifiability guarantees of causal discovery methods across ANM assumptions.
\myyes~denotes full DAG identifiability; \myno~denotes no guarantee.
LiNGAM~\citep{shimizu2006linear, shimizu2011directlingam} only deals with linear non-Gaussian ANMs, SCORE~\citep{rollandScoreMatchingEnables2022} requires nonlinear mechanisms with Gaussian noise.
NoGAM~\citep{montagnaScoreMatchingRoof2025} extends SCORE to arbitrary noise distributions.
CAM-UV~\citep{maeda2021causal} additionally handles latent confounders under nonlinear separable (CAM) assumptions. \textcolor{black}{SCORE can handle latent hard selection bias in both Gaussian and non-Gaussian cases, which is not discussed in the original paper but is one of the surprising findings in this work (\myyes$^{*}$)}. Our method (CHess) is the only method that is resilient to hard selection under all the above ANMs.  \textcolor{black}{Linear Gaussian models are known to be not identifiable beyond the Markov equivalence class~\citep{peters2014identifiability} thus omitted in the table.}
}
\label{tab:identifiability}
\end{table}
\vspace{-10pt}
\section{Preliminaries and Problem Setup}\label{sec:problem}

\subsection*{Problem Setting: Hard Selection and What We Know about it}\label{sec:setting}
We consider a set of random variables $\X = \{X_1, \ldots, X_d\}$ generated by a restricted additive noise model (ANM):
\begin{equation}\label{eq:anm}
    X_i := f_i(X_{\PA_i}) + N_i, \qquad i = 1, \ldots, d,
\end{equation}
where $f_i$ are $C^2$ causal mechanisms, $N_i$ are mutually independent noise terms with strictly positive densities, and the model satisfies the restricted ANM identifiability condition~\citep{buhlmannCAMCausalAdditive2014}. The associated causal graph $\cG$ is a DAG.

In the remainder of the paper, we assume the Markov and faithfulness conditions for the DAG $\cG$ and the distribution $p(\X)$~\citep{spirtes2000causation}. We say two nodes $X_i$ and $X_j$ are \emph{d-separated} by a set of nodes $\mathbf{Z}$ in $\cG$ if every path between them is blocked by $\bZ$, and a path is blocked by conditioning on $\bZ$ if the path contains a non-collider node in $\bZ$ or a collider node that is not in $\bZ$ and has no descendant in $\bZ$. For a full definition, please refer to the Appendix.
\paragraph{Hard Selection (Truncation)}
We observe biased data from $p(\X \mid S = 1)$,  where $p(\X \mid S=1)\neq p(\X)$.  The selection mechanism is unknown and can be dependent on an arbitrary subset of $\X$. Here we focus on the hard selection (truncation) case, which is the scenario where the selection mechanism is deterministic and only data points in a subset $\cB$ of the support of $\X$ are observed, i.e. $p(S = 1 \mid \X=x) = \mathbbm{1}(x \in \cB)$ for a set $\mathcal{B}\subseteq \bbR^d$ with $\operatorname{vol}(\mathcal{B}) > 0$. 

\paragraph{Causal Discovery under Selection Bias} Given the biased observational data $\cD = \{x^{(i)}\}_{i=1}^N$ from $p(\X \mid S = 1)$, the goal is to recover the causal graph $\cG$ under the original distribution $p(\X)$. This is a challenging problem that has been studied in the causal discovery literature in local causal discovery~\citep{versteegLocalConstraintBasedCausal2022a}, interventional settings~\citep{daiWhenSelectionMeets2025, luo2025gene}, latent variable settings~\citep{daiLatentVariableCausal2025}, etc. However, to the best of our knowledge, \emph{no method has been developed in the general setting} to tackle selection bias for causal discovery only with observational data. For a comprehensive review of the related work, please refer to Appendix~\ref{sec:related}.

Given the generality of the problem, we focus on the following question in this paper: 
\begin{examplebox}
\emph{Can we find properties of the original distribution $p(\X)$ preserved under hard selection, and how can we leverage them for robust causal discovery?} 
\end{examplebox}
We take the first step to answer this question by revisiting the score function and score-matching based causal discovery methods, and show its potentials.
\section{The Blessing of Score Functions}\label{sec:score}
In this section,  we look at the score function (Section~\ref{sec:score-preserve}) and review score-matching for causal discovery (Section~\ref{sec:score-revisit}).  Then in Section~\ref{sec:motivation}, we present the advantage of using the score directly for causal discovery under truncation, and an example showing a failure case, which motivates our proposed method in the next section (Section~\ref{sec:method}).

\subsection{Score Preservation}\label{sec:score-preserve}
Consider the score of the selected distribution $\nabla_X \log p(x\mid s)$. The selected density satisfies
\[
    \log p(x \mid S = 1) = \log p(x) + \underbrace{\log \mathbbm{1}(x \in \cB)}_{= 0 \text{ for all $x \in$} \operatorname{int}(\cB)} - \underbrace{\log p(S=1)}_{\text{constant}}, \quad \forall x \in \bbR^d,
\]
because of Bayes' rule. Take the derivative, for $x \in \text{int}(\mathcal{B})$, 
\begin{equation}\label{eq:score-preserved}
    \boxed{\underbrace{\nabla_X \log p(x \mid S = 1)}_\text{(the score on the biased data)}  = \underbrace{\nabla_X \log p(x)}_\text{(the score on the true data)}}
\end{equation}
The score is exactly preserved at every observed data point.
\begin{remark} [Score preservation under hard selection]
    This is an important and nontrivial observation that we leverage in this paper for developing a novel score-based method for causal discovery under hard selection. The same argument also applies to the second-order score $\nabla^2_X \log p(x)$, which is also preserved under hard selection, i.e. $\nabla^2_X \log p(x \mid S = 1) = \nabla^2_X \log p(x)$ (in the interior of $\cB$). This aligns with the local geometrical nature of the score, which characterizes the curvature of the log-density pointwise, and which is preserved under truncation. 
\end{remark}

\begin{remark}[Boundary effect]\label{rem:boundary}
    One might notice that the score is only preserved in the interior of $\cB$, and not on its boundary, where we have an extra term $\delta(x) $ in the score due to the derivative of the indicator function. This boundary effect can be problematic for score estimation: \citep{rollandScoreMatchingEnables2022,montagnaScoreMatchingRoof2025} show that we can use Stein identity \citep{stein1972bound} for score estimation, which relies on an vanishing boundary condition on $p(x)$, i.e. $h(x)p(x) \to 0$ as $x \to \infty$, for any testing function $h(x)$. However, under truncation, the selected density $p(x \mid s)$ does not satisfy this condition. In practice, we propose to use the diffusion score matching method~\citep{sanchezDiffusionModelsCausal2023}, which is more robust to boundary effects. We will discuss this in more detail in Section~\ref{sec:method}.
\end{remark}

\subsection{Revisiting Score-Matching}\label{sec:score-revisit}
We provide a review of the score-matching based causal discovery methods. The score component for node $X_i$ decomposes as:
\begin{equation}\label{eq:score-decomp}
    \partial_{X_i} \log p(x) = s_{n_i}(N_i) - \sum_{k \in X_{\CH_i}} \frac{\partial f_k}{\partial x_i}(x_{\PA_k}) \cdot s_{n_k}(N_k),
\end{equation}
where $s_{n_i}(N_i) := (\log p_{N_i})'$ is the score function of the $i$-th noise distribution and $X_{\CH_i}$ denotes the children of $X_i$ in $\cG$.
\paragraph{Sink Identification}
For a sink node $X_l$ (with $X_{\CH_l} = \emptyset$):
\begin{equation}\label{eq:sink-score}
    \partial_{X_l} \log p(x) = s_{n_l}(N_l).
\end{equation}

Under \emph{Gaussian} noise, \cite{rollandScoreMatchingEnables2022} propose the variance test (Proposition~\ref{prop:var-test}) as sink node identification method criterion: 
\begin{proposition}[Gaussain case: Variance test~\citep{rollandScoreMatchingEnables2022} ]\label{prop:var-test} Let $\X$ be generated by a restricted ANM~\eqref{eq:anm} with mutually independent Gaussian noise terms.
    \begin{equation}
        \Var\left[\partial_{X_i} {\color{black}s_{x_i}}(\mathbf{X})\right]=0 \Longleftrightarrow X_i \text { is a sink node, } \forall i=1, \ldots, d .
    \end{equation}
\end{proposition}

For a sink node, $s_{x_l}(x_l) = -\frac{x_l - f_l(\PA_l(x_l))}{\sigma_l^2}$, and  $s_{x_l}'(x_l) = c.$ Therefore, $\Var_X(s'_{n_l}(N_l)) = 0$. 

Extending beyond Gaussian noise, \cite{montagnaScoreMatchingRoof2025} construct the regression residual test (Proposition~\ref{prop:res-test}) as sink-node identification method criterion:

\begin{proposition}[Non-Gaussain case: Residual test~\citep{montagnaScoreMatchingRoof2025}]\label{prop:res-test} Let $\X$ be generated by a restricted ANM~\eqref{eq:anm} with mutually independent non-Gaussian noise terms. For any $i \in [d]$,
    \begin{equation}
        \bbE\left[\left(\partial_{X_i} {\color{black}s_i}(\mathbf{X})-\bbE\left[\partial_{X_i} {\color{black}s_i}(\mathbf{X}) \mid R_i\right]\right)^2\right]=0 \Longleftrightarrow X_i \text { is a sink node. }
    \end{equation}
\end{proposition}

For a sink node, $R_l = X_l - \bbE[X_l \mid X_{\setminus l}]$, which equals the centered noise $N_l-\bbE(N_l)$ for a true sink node. This means for a sink node (without selection bias), the score of $X_l$ can be perfectly predicted from the residual $R_l$. 


\paragraph{Edge (neighbor) Identification} 
 The Markov equivalence graph can be constructed by reading the conditional independence of pairs of random variables as shown by~\citep{spantini2018inference, montagnaScalableCausalDiscovery2023a}, in Proposition~\ref{prop:edge-test}.
 \begin{proposition}[~\citep{spantini2018inference, montagnaScalableCausalDiscovery2023a}]~\label{prop:edge-test}
        Let $\X$ be generated by a restricted ANM~\eqref{eq:anm} with mutually independent noise terms. For any $i, j \in [d]$,
        \begin{equation}
            \partial^2_{X_i X_j} \log p(x) = 0 \Longleftrightarrow X_i \ind X_j \mid \X_{\setminus \{i, j\}} \text{in } \cG.
        \end{equation}
\end{proposition}

\subsection{When Score-Matching is Surprisingly Resilient and When it is Not}\label{sec:motivation}
\paragraph{How does score-matching behave under hard selection?} In this subsection, we show that for sink-node identification under truncation, Proposition~\ref{prop:var-test} and~\ref{prop:edge-test} still hold (modified as Proposition~\ref{prop:var-test-trunc} and ~\ref{prop:edge-test-trunc}), but Proposition~\ref{prop:res-test} fails in general. Unlike other causal discovery methods that rely on conditional independence tests that are susceptible to hidden selection variables, the score-matching based methods are surprisingly resilient to hard selection bias. However, the residual test relies not only on the score, but also on the structure underlying the density, which can be distorted by selection bias. This motivates us to leverage the score directly for causal discovery, without relying on the residual or any graphical assumptions.

\paragraph{\underline{The first fold: Variance test still holds}}
Under nonlinear ANM and Gaussian noise assumptions, the result in Proposition~\ref{prop:var-test} can be directly extended to truncated data, as stated in Proposition~\ref{prop:var-test-trunc}.
\begin{proposition}[Gaussian case: Variance test under truncation]\label{prop:var-test-trunc}Let $\X$ be generated by a restricted ANM~\eqref{eq:anm} with mutually independent Gaussian noise terms, and let $S = \mathbbm{1}(x \in \cB)$ denote
a deterministic selection  for a truncation region $\cB$.
    \begin{equation}
        \Var\left[\partial_{X_i} s_i(\mathbf{X}\mid {\color{blue}S=1})\right]=0 \Longleftrightarrow X_i \text { is a sink node, } \forall i=1, \ldots, d .
    \end{equation}
\end{proposition}

This implies that with data from $p(\X \mid S = 1)$, we can still identify the sink nodes by checking the variance of the Hessian log-density, which is an additional outcome not realized before. The proof directly relies on the score preservation property in ~\eqref{eq:score-preserved}.

{Similarly, the result from Proposition~\ref{prop:edge-test} naturally extend to testing $ \partial^2_{X_i X_j} \log p(x \mid S=1)$ for skeleton discovery, as shown in Appendix~\ref{sec:proofs}.}
\paragraph{\underline{The second fold: Residual test fail}}
    Under hard selection, the selected (observed) residual is:
\[
    R_l^{\mathrm{sel}} = X_l - \bbE[X_l \mid X_{\setminus l},\, S = 1].
\] 

{If the selection is only a function of a single node, then the biased residual $R_l^{\mathrm{sel}}$ is still a function of $X_l$. This is a special case where the residual test still holds, which is elaborated in Appendix~\ref {sec:proofs}. But in general, when we do not know whether the selection is only a function on a single node or not, the residual test fails as in Example~\ref{ex:fail}.}
\begin{example}[Failure case]\label{ex:fail}
    Consider a sink node $X_l$ with parent $X_{\PA_l}$. If the selection mechanism depends on both $X_l$ and some $X_k\in X_{\PA_l}\subseteq X_{\setminus l}$ (e.g., $S = \mathbbm{1}(|X_l + X_k| \leq c)$), conditioning on $S = 1$ induces dependence between $N_l$ and $X_{\PA_l}$, so $R_l^{\mathrm{sel}} = X_l - \bbE[X_l \mid X_{\setminus l},\, S = 1]\neq X_l - \bbE[X_l \mid X_{\setminus l}]= R_l = N_l-\bbE N_l$. Note that the score $\partial_{X_l} \log p(x) = s_l(N_l)$ is preserved, but the residual $ R_l^{\mathrm{sel}}$ is a biased proxy for $N_l$, and the residual test incorrectly rejects the true sink.
\end{example} 

\paragraph{How do we surpass the issue? Our proposal: exploit the score directly for causal asymmetry.} The residual test fails because it depends on the global structure of the density, while the score is a local property of the density at each data point, and is preserved under truncation~\eqref{eq:score-preserved}. Therefore, we need to exploit this property of the score directly for a novel method for causal discovery. Since we have seen that Proposition 3 remains true under selection, the remaining challenge is to identify the sink nodes given the skeleton.

\section{Causal Discovery Method Robust to Hard Selection}\label{sec:method}
\newcommand{\IntuiPic}{\begin{tikzpicture}[
    every node/.style={font=\large},
    arr/.style={->, thick, >=stealth}
]

\node (Np) at (3, 1.5)   {\textcolor{orange}{$N_p$}};
\node (Xp) at (2, 1)   {$X_p$};
\node (Xc) at (2, 0)   {$X_c$};
\node (Nc) at (3, 0.5)   {\textcolor{orange}{$N_c$}};
\node (dots1) at (1, 2) {$\cdots$};
\node (dots2) at (0.6, 1) {$\cdots$};

\draw[arr, orange]        (Np)    -- (Xp);
\draw[arr]                (dots1) -- (Xp);
\draw[arr]                (Xp)    -- (Xc);
\draw[arr, orange]        (Nc)    -- (Xc);
\draw[arr]                (dots2)    -- (Xp);

\end{tikzpicture}}
We present our main result in three-fold: we first introduce the intuition of the method in Section~\ref{sec:intuition}, then in Section~\ref{sec:chess} we give the main theoretical result under linear ANM assumptions (with non-Gaussian noise) in Proposition~\ref{prop:rank-cond-L} and the more general result for non-linear ANM in Proposition~\ref{prop:rank-cond-nL}. 


\subsection{Intuition and Definitions}\label{sec:intuition}
As shown in the example in Section~\ref{sec:motivation}, the residual test \citep{montagnaScoreMatchingRoof2025} fails under hard-selected data in general (in non-Gaussian noise cases). 

\vspace{-3mm}
\begin{wrapfigure}{r}{0.25\textwidth}
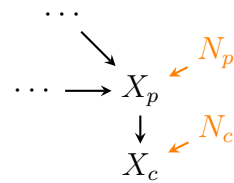

\vspace{-10pt}
    \centering
    \IntuiPic
    \caption{In the subgraph, the sink node $X_c$ is only d-connected to its own noise $N_c$ after conditioning on $X_{\PA_c}$, while $X_p$ is d-connected to $N_p$ and $N_c$ after conditioning on $X_{\PA_p} \cup X_{\CH_p}$.}\label{fig:graph_noise}
    \vspace{-2.5em}
\end{wrapfigure}

\looseness=-1\paragraph{Intuition: Asymmetry of Number of Independent Noise Sources.} The key intuition is that the score of a sink node has exactly \emph{one} independent noise term that contributes to it, after conditioning on all its neighbors\footnote{In practice (Algorithm~\ref{alg:main}), we perform skeleton discovery first to identify the neighbors of each node.}, while a non-sink node has $\geq 2$ independent noise terms. Consider the graph in Figure~\ref{fig:graph_noise}, the path between a non-sink node and the noise of its child is opened by conditioning on its child, while the argument does not apply to a sink node that does not have child terms.

\paragraph{Algebraic Connection: } To translate the number of independent noises into an algebraic criterion, we consider the score function $s(x)\coloneqq \nabla \log p(x)$. The main observation is that if we take the gradient of the score on node $X_i$ w.r.t all noise terms\footnote{Could be understood as taking the gradient w.r.t. all nodes with a change of basis}, we obtain a vector that has contributions from multiple source directions. The conditional rank (conditioning on all neighbors of $X_i$) of the variance of such a vector's outer product indicates the number of independently varying factors, e.g., for a sink node the rank is $1$ (for non-Gaussian noise), which we use as an identification rule to distinguish a sink node from a non-sink node. 

In light of the intuition and algebraic connection above, we formally define the Conditional Hessian ($M_i^{\mathrm{cond}}$) outer product matrices in Definition~\ref{def:HM}.  
\begin{definition}[Conditional Hessian outer product matrix]\label{def:HM}
For each node $X_i$, define the $i$-th Hessian row vector evaluated at $x$:
\begin{equation}\label{eq:hessian-row}
    H_i(x) := \left(\frac{\partial^2 \log p(x)}{\partial X_i \partial X_1},\; \ldots,\; \frac{\partial^2 \log p(x)}{\partial X_i \partial X_d}\right) \in \bbR^d.
\end{equation}

We define the \emph{conditional gradient outer product matrix} as:
\vspace{-3pt}
\begin{equation}\label{eq:Mjcon}
    M_i^{\mathrm{cond}} \rev{(z)}\coloneqq
  \Var\bigl(H_i(\mathbf{X}) \,\big|\, X_{\NB_i}\rev{=z}\bigr)
 \in \mathbb{R}^{d \times d},
\end{equation}
where $X_{\text{NB}_i}$ denotes the set of neighbors of $X_i$.

\end{definition}


\paragraph{Hessian Row Structure: sinks vs.\ Non-sinks}

The Hessian row $H_i$ is the gradient of the score component $s_i(x) = \partial_{X_i} \log p(x)$. By the chain rule applied to \eqref{eq:score-decomp}:
\begin{equation}\label{eq:Hj-decomp}
    H_i(x) = \underbrace{s_i'(N_i) \cdot \nabla_x N_i}_{\circled{1}} \;-\; \sum_{k \in X_{\CH_i}} \Big[\underbrace{s_k'(N_k) \cdot \frac{\partial f_k}{\partial x_i} \cdot \nabla_x N_k}_{\circled{2}\text{child noise direction}} + \underbrace{s_k(N_k) \cdot \nabla_x\!\left(\frac{\partial f_k}{\partial x_i}\right)}_{\circled{3}\text{mechanism curvature}}\Big],
\end{equation}
where $N_i  = x_i - f_i(x_{\PA_i})$ is the noise term at node $X_i$, so $\nabla_x N_i = e_i - \nabla_{x} f_i(x_{\PA_i})$ and $e_i$ is the $i$-th standard basis vector. For a sink node, the whole summation term over all children vanishes, so $H_i$ only has one noise term direction contributed by $N_i$:
\begin{equation}\label{eq:Hj-decomp-sink}
    H_l(x) = s_l'(N_l) \cdot \nabla_x N_l \quad (\text{only term } \circled{1} \text{ remains}).
\end{equation}
Such direction $\nabla_x N_i$ is a constant vector if the mechanism is linear, or a varying vector function of $x$. 

\begin{examplebox}
    In the following section, we will formalize the intuition for the linear ANM case first: under linearity, $\circled{3}$ vanishes and non-sink node $H_i$ has $|\CH_i| + 1$ noise term directions contributed by $N_i$ in \eqref{eq:Hj-decomp}, while sink node $H_l$ only has one noise term direction contributed by $N_l$ in \eqref{eq:Hj-decomp-sink}.
\end{examplebox}
\subsection{Main Results: from Linear to Nonlinear ANM}\label{sec:chess}
First, we present the result under linear ANM with non-Gaussian noise, which clearly illustrates the intuition of the method and the algebraic connection between the number of independent noise terms and the rank of $M_i^{\mathrm{cond}}$. Then we present the more general result for nonlinear ANM with both Gaussian and non-Gaussian noise.

For a linear ANM, the decomposition \eqref{eq:Hj-decomp} simplifies to a combination of noise term directions (only terms $\circled{1}$ and $\circled{2}$ remain, and $\nabla N_j, j \in \{i\}\cup \CH_i$ is a constant vector). The rank of $M_i^{\mathrm{cond}}$ then directly counts the number of independent noise terms that contribute to the variance of $H_i$ after conditioning on the neighbors of $X_i$: conditioning on the parents of $X_i$ blocks the contribution of the noise terms from the upstream (parents), while conditioning on the children of $X_i$ still keeps the contribution of the noise terms of the children.
\begin{propbox}
\begin{restatable}[Connection between Conditional Rank and Independent Noise Terms in Linear ANM]{proposition}{RANKCONDLIN}\label{prop:rank-noise-count}
Under a linear non-Gaussian ANM, the rank of $M_i^{\mathrm{cond}}\rev{(z)}$ equals
the number of independent noise terms contributes to $s_i$ by conditioning on $X_{\NB_i}$ \rev{for generic $z$\footnote{By genericity, we exclude the degenerate cases where the graph structure is not faithfully reflected in the rank property, similar to the cases where standard faithfulness is violated. See Remark. }}:
\[
  \rank\bigl(M_i^{\mathrm{cond}}\rev{(z)}\bigr) \;=\; |\mathcal{N}_i|,
\]
where the \emph{noise index set} $\mathcal{N}_i \;:=\; \{i\} \cup \CH_i$.

\end{restatable}
\end{propbox}
\looseness=-1\textbf{Proof Sketch.} The proof consists of two main steps: (1) express $M_i^{\mathrm{cond}}$ as a sum of rank-1 outer products $\mathbf{v}_j \mathbf{v}_j^\top$ weighted by the variance of the noise terms; (2) show that the vectors $\mathbf{v}_j$ corresponding to different noise terms are linearly independent, so the rank of the sum equals the number of noise terms.

To illustrate the intuition of Proposition~\ref{prop:rank-noise-count}, we present the following example of a chain structure, where the sink node has only one independent noise term while the non-sink node has two independent noise terms after conditioning on its neighbors.
\begin{restatable}[Number of Independent Terms in Chain]{example}{CHAIN}\label{ex:chain}
Consider a chain $X_1\rightarrow X_2\rightarrow X_3$, where 
$X_1 = N_1$, $X_2 = aX_1 + N_2$, $X_3 = bX_2 + N_3$ with non-Gaussian noise.
\end{restatable}
\begin{proof}[Illustration of Example~\ref{ex:chain}]
For non-sink $X_1$ and $X_2$, both have two independent noise terms that contribute to their score function after conditioning on their neighbors, thus
the rank of the conditional outer product matrix $M_1^\mathrm{cond}$ and  $M_2^\mathrm{cond}$ (Eqn~\eqref{eq:Mjcon}) is $2$, while  $M_3^\mathrm{cond}$ has rank $1$. The reason is that, for source node $X_1$, conditioning on its neighbor (child in this case) $X_2=z_2$ induces dependency of $N_2$ on $N_1$, thus leading to the direction of $N_1$ and a non-affine transformation of $N_1$, while for a sink node $X_3$, conditioning on its neighbor (parent in this case) will only leave a single direction term contributed by $N_3$. A formal proof is provided in Appendix~\ref{app:example}
\end{proof}

\vspace{-5pt}
\begin{propbox}
\begin{restatable}[Conditional Rank Test for Sink Identification (Linear)]{proposition}{RANKCONDL}\label{prop:rank-cond-L}
    Let $\X$ be generated by a restricted \emph{linear} ANM~\eqref{eq:anm} with independent noise terms. Then \rev{for generic $z$}:
    \begin{enumerate}[label=(\alph*),nosep]
        \item \textbf{Sink.} If $X_l$ is a sink (i.e., $\CH_l = \emptyset$), then  $\rank \!\bigl(M_l^{\mathrm{cond}}\rev{(z)}\bigr)= 1$.
        \item \textbf{Non-sink.} If $X_i$ is not a sink (i.e., $|\CH_i| \geq 1$), then  $\rank \!\bigl(M_i^{\mathrm{cond}}\rev{(z)}\bigr) \geq  2$.
    \end{enumerate}
    Therefore, a node $X_i$ is a sink if and only if $\rank (M_i^{\mathrm{cond}}\rev{(z)}) = 1$. 
\end{restatable}
\end{propbox}
\vspace{-7pt}

\begin{proof}[Proof sketch] This is a direct consequence of Proposition~\ref{prop:rank-noise-count} and the fact that the rank of $M_i^{\mathrm{cond}}$ is invariant to conditioning on $X_{\NB_i}$ (since the noise terms are independent of $X_{\NB_i}$, so the variance of $s'_j(N_j)$ does not change after conditioning).
\end{proof}

\vspace{-5pt}
More generally, for the nonlinear ANM,  \emph{the rank gap between sink and non-sink nodes still exists}, because the decomposition \eqref{eq:Hj-decomp} contains an additional mechanism curvature term, and this term is generally linearly independent of the noise term directions. Due to space limit, we present {the informal Proposition~\ref{prop:rank-cond-nL-informal} and the full proposition under nonlinear ANM in the Appendix (Proposition~\ref{prop:rank-cond-nL}).
\begin{proposition}[Informal Conditional Rank Test for Sink Identification (Nonlinear)]\label{prop:rank-cond-nL-informal}
    Let $\X$ be generated by a restricted \emph{nonlinear} ANM~\eqref{eq:anm} with independent noise terms. Then the rank gap of $\rank(M_i^{\mathrm{cond}}\rev{(z)})$ \rev{for generic $z$} between sink and non-sink node still exists {and is at least one}.
\end{proposition}}
For a summary of the rank values for sink vs.\ non-sink nodes across different ANM settings, see Table~\ref{tab:rank-cond}. In particular, the rank gap between sink and non-sink nodes exists (thus sink nodes are identifiable) in all settings except for the linear Gaussian case, where all nodes have rank $0$.



\begin{table}[htbp]
\centering
\renewcommand{\arraystretch}{1.3}
\begin{tabular}{@{}lccl@{}}
\toprule
{Setting} & {Sink rank} & {Non-sink rank} & {Gap} \\
\midrule
\textcolor{black}{Linear  Gaussian} & \textcolor{black}{$0$} & \textcolor{black}{$0$} & \textcolor{black}{$0$ (not identifiable)} \\

Linear  non-Gaussian & $1$ & $\geq 2$ & $ \geq 1$\\
Nonlinear  Gaussian & $0$ & $\geq 1$ & $\geq 1$ \\
Nonlinear  non-Gaussian & $1$ & $\geq 2$ & $ \geq 1$ \\
\bottomrule
\end{tabular}
\vspace{5pt}
\caption{Rank of $M_i^{\mathrm{cond}}\rev{(z)}$ for sink vs.\ non-sink nodes across different ANM settings.}\label{tab:rank-cond}
\end{table}

{\begin{remark}[Practical Usage of Proposition~\ref{prop:rank-cond-L} and Table~\ref{tab:rank-cond}] Without knowledge about the type of ANM and type of noise one is dealing with, a $\rank(M_i^\text{cond})=1$ result could either imply a sink node in a linear non-Gaussian model, or a non-sink node in a nonlinear Gaussian model. In practice, for either case, we do not bother to distinguish the model / noise types, but only take the node that has the minimal rank to be the current sink node (and we perform this procedure iteratively), i.e., as long as a gap exists.\end{remark} }

\vspace{-5pt}

\paragraph{Discussion of Unconditional Hessian Rank Test}
Why do we need the conditioning on the neighbors? Can we just use the unconditional Hessian outer product matrix $M_i = \bbE[H_i(\X) H_i(\X)^\top]$ to identify sink nodes? The short answer is yes, but it requires strong structural constraints. {In Appendix~\ref{sec:proofs:uncond}, we introduce the \emph{Unconditional Rank Test}. However, it requires stronger structural assumptions (involving the number of $|\PA_i|$ and  $|\CH_i|$) since the noise coming from the parents is not blocked by conditioning on the neighbors anymore. On the other hand, the Conditional Hessian Rank Test proposed in Proposition~\ref{prop:rank-cond-L} is free from the structural assumptions and thus more practical.}



\section{Estimation}\label{sec:estim}
In this section, we present a practical estimation procedure for causal discovery based on the conditional Hessian rank test. This procedure involve $4$ main steps:(1) estimation of Hessian log-density ($H_j(x^{(i)})$), (2) use the estimated Hessian $H_j(x^{(i)})$ to estimate the conditional variance matrix $\widehat{M}_j^{\,\mathrm{cond}}$ , (3) compute the  rank of $\widehat{M}_j^{\,\mathrm{cond}}$ to identify sinks, and (4) . The overall algorithm is provided in Algorithm~\ref{alg:main} and more details on the estimation procedure in Appendix~\ref{sec:score-estim}.
\paragraph{Score and Hessian Estimation} 
Estimation of the score function (and its Jacobian) is a well-studied problem in statistics and machine learning, e.g., the Stein estimator~\citep{stein1972bound,rollandScoreMatchingEnables2022,montagnaScoreMatchingRoof2025}, the kernel density estimator (KDE)~\citep{sasaki2016direct}, and most recently a diffusion model based estimator~\citep{sanchezDiffusionModelsCausal2023}. In this paper, we apply the commonly used Stein-estimator, and also a diffusion model based estimator, because it does not rely on a strictly positive density $p(x)$ while both Stein and KDE estimators suffer when $p(x)\approx 0$ (See Remark~\ref {rem:boundary}).

\begin{remark} (Practical Estimation of $M_i^\text{cond}$)
    In practice, we use a residual estimator $\widehat{M}_i^{\text{cond}} := \frac{1}{n}\sum_{t=1}^n r_i(x^{(t)}) r_i(x^{(t)})^\top$ as a proxy, where $r_i(x) := H_i(x) - \hat{m}_i(x_{\text{NB}_i})$ is the residual of using  $\hat{m}_i(\cdot)$ to estimate $H_i$. The residual estimator $\widehat{M}_i^{\text{cond}}$ shares a similar spirit with the estimation method in NoGAM~\citep{montagnaScoreMatchingRoof2025}, which also uses the residual of a regression to estimate the noise. Further details are discussed in the Appendix.~\ref{sec:app:estim}
\end{remark}

\section{Experiments}\label{sec:exp}
Given the proposed method, we conduct experiments to evaluate its performance on both synthetic (Section~\ref{sec:synthetic}) and real-world (Section~\ref{sec:realdata}) datasets. The goal is to show that ($1$) for general causal discovery, the proposed method performs competitively with state-of-the-art baselines, and ($2$) in the presence of hard selection bias, the proposed method outperforms all baselines that do not explicitly account for selection bias.

\subsection{Synthetic Experiments}\label{sec:synthetic}
\paragraph{Data generation.}
We generate Erd\H{o}s--R\'enyi DAGs with $d \in \{5, 10\}$ nodes and target average degree $2.0$. 
We consider three SEM regimes: ($1$)Linear non-Gaussian (Linear-nG).
($2$){Nonlinear Gaussian (MLP-G): }$X_i = f_i(\mathrm{Pa}(X_i)) + \epsilon_i$, with $\epsilon_i\sim\mathcal{N}(0,\sigma_i^2)$ and $f_i$ a randomly initialised $1$-hidden-layer MLP.
 ($3$){Nonlinear non-Gaussian (MLP-nG).} Same MLP mechanism as MLP-G but with Gumbel or log-normal noise as in Linear-nG.

\paragraph{Selection mechanism.}
We add $K$ latent selection variables $S_1,\dots,S_K$. We simulate the hardest case of selection bias where the selection mechanism depends on \emph{all} observed variables. Each $S_k$ is a random linear combination $S_k= \sum_{j=1}^d \alpha_{kj}\, X_j$ with coefficients $\alpha_{kj}\!\sim\!\mathcal{U}(-1, 1)$, and we keep the samples where $S_k \geq q_k$ for all $k$ simultaneously, where $q_k$ is the $1 - \rho$ empirical quantile, $\rho$ being the retaining fraction. After applying the selection mechanism, we keep sample sizes $n=1000$ (from an initial pool of $10^5$ samples) \rev{to avoid the interference of different sample sizes on the performance}.


\looseness=-1\textbf{Baselines. \ }
We compare against constraint-based methods PC and FCI~\citep{spirtes2000causation,spirtesCausalInferencePresence2013a}, score-based methods GES~\citep{chickering03_ges}, permutation-based method GRaSP~\citep{lam2022greedy}, and LiNGAM-family estimators ICA-LiNGAM~\citep{shimizu2006linear}
and DirectLiNGAM~\citep{shimizu2011directlingam}, together with score-matching baselines SCORE~\citep{rollandScoreMatchingEnables2022}, DAS~\citep{montagnaScalableCausalDiscovery2023a}, NoGAM~\citep{montagnaScoreMatchingRoof2025},
CAM~\citep{buhlmannCAMCausalAdditive2014}, and DiffAN~\citep{sanchezDiffusionModelsCausal2023}. We report F1 of the arrow on the estimated DAG with part of the baselines due to space limit\footnote{For details of the evaluation metrics (SHD, F1 adjacency, ordering accuracy, FPR), and results on all of the baselines, see Appendix~\ref{sec:expresults}.}, with one selection variables as a function of all observed variables. We also include sensitivity analysis with varying selection fraction $\rho \in \{1.0,0.1, 0.05\}$, selection shape $\{$\verb+halfspace+, \verb+annulus+, \verb+two boxes+ $\}$ in Appendix~\ref{sec:expresults}.

\paragraph{Results.} In the absence of selection bias ($\rho=1.0$), the proposed method performs competitively with the baselines across all SEM regimes, and outperforms all baselines when hard selection bias is present (in the linear-nG regime, the LiNGAM baselines sometimes outperform even when selection bias is present, which is discussed in Appendix~\ref{sec:expresults}). 

\begin{table}[htbp]
\centering
\caption{F1 (Arrow) for Linear-nG (gumbel), $d=5$ and $d=10$}
\label{tab:f1_arrow_gumbel_combined}
\renewcommand{\arraystretch}{1.2}
\setlength{\tabcolsep}{3pt}
\begin{tabular}{l|ccc|ccc}
\toprule
 & \multicolumn{3}{c|}{$d=5$, gumbel} & \multicolumn{3}{c}{$d=10$, gumbel} \\
fraction & $\rho=1.0$ & $\rho=0.1$ & $\rho=0.05$ & $\rho=1.0$ & $\rho=0.1$ & $\rho=0.05$ \\
\midrule
SCORE & $0.57$ \var{0.25} & $0.37$ \var{0.17} & $0.41$ \var{0.14} & $0.65$ \var{0.12} & $0.46$ \var{0.11} & $0.45$ \var{0.10} \\
DAS & $0.57$ \var{0.25} & $0.37$ \var{0.17} & $0.41$ \var{0.14} & $0.66$ \var{0.11} & $0.48$ \var{0.12} & $0.49$ \var{0.10} \\
NoGAM & $0.81$ \var{0.21} & $0.40$ \var{0.16} & $0.40$ \var{0.14} & $0.78$ \var{0.16} & $0.47$ \var{0.14} & $0.45$ \var{0.11} \\
CAM & $0.31$ \var{0.20} & $0.18$ \var{0.14} & $0.19$ \var{0.12} & $0.23$ \var{0.08} & $0.22$ \var{0.11} & $0.19$ \var{0.12} \\
DirectLiNGAM & $\mathbf{0.99}$ \var{0.03} & $0.64$ \var{0.14} & $0.65$ \var{0.16} & $\mathbf{0.98}$ \var{0.03} & $0.56$ \var{0.05} & $0.54$ \var{0.07} \\
DiffAN & $0.40$ \var{0.28} & $0.31$ \var{0.24} & $0.29$ \var{0.23} & $0.33$ \var{0.15} & $0.25$ \var{0.16} & $0.25$ \var{0.14} \\
CHess (Stein) & $0.90$ \var{0.14} & $\mathbf{0.76}$ \var{0.11} & $0.70$ \var{0.16} & $0.72$ \var{0.19} & $\mathbf{0.59}$ \var{0.14} & $0.51$ \var{0.09} \\
CHess (DiffAN) & $0.81$ \var{0.25} & $0.71$ \var{0.17} & $\mathbf{0.72}$ \var{0.11} & $0.64$ \var{0.15} & $0.49$ \var{0.15} & $\mathbf{0.64}$ \var{0.09} \\
\bottomrule
\end{tabular}
\end{table}

\subsection{Real-world experiments}
\label{sec:realdata}

\paragraph{Dataset.}
We use the protein-signalling dataset of \citet{sachs2005causal},
loaded through the \texttt{cdt} package, which yields $7466$ samples over $d=11$ proteins and a verified ground-truth DAG.  To probe robustness to selection bias in severe cases, we apply the selection mechanism as in Section~\ref{sec:synthetic} with $K=1, \rho=0.1$. \rev{After selection, there are $373$ samples, so we apply random downsampling to $n=373$ samples for the non-selection bias setting $\rho=1.0$ to avoid the interference of different sample sizes with pruning pamareter for a fair comparison.}



\begin{wraptable}{r}{0.5\textwidth}
\vspace{-2em}
\centering
\caption{F1 (Arrow) on Sachs Dataset for $\rho=1$ (downasmpled) and $\rho=0.1$}
\vspace{-3pt}
\label{tab:f1_arrow_extracted}
\renewcommand{\arraystretch}{1.2}
\setlength{\tabcolsep}{3pt}
\begin{tabular}{l c c}
\toprule
Method & $\rho=1$ ($n=373$) & $\rho=0.1$\\
\midrule
SCORE & $0.28$ \var{0.05} & $0.29$ \var{0.05} \\
DAS & $0.19$ \var{0.10} & $0.23$ \var{0.08} \\
NoGAM & $0.29$ \var{0.04} & $0.30$ \var{0.06} \\
CAM & $0.40$ \var{0.08} & $0.28$ \var{0.09} \\
PC & $0.25$ \var{0.09} & $0.23$ \var{0.07} \\
FCI & $0.13$ \var{0.04} & $0.14$ \var{0.11} \\
GES & $0.25$ \var{0.06} & $0.22$ \var{0.08} \\
GRaSP & $0.28$ \var{0.07} & $0.22$ \var{0.10} \\
ICA-LiNGAM & $0.27$ \var{0.08} & $0.30$ \var{0.07} \\
DirectLiNGAM & $0.30$ \var{0.07} & $0.33$ \var{0.09} \\
DiffAN & $0.30$ \var{0.08} & $0.19$ \var{0.05} \\
CHess (Stein) & $\mathbf{0.42}$ \var{0.09} & $0.32$ \var{0.12} \\
CHess (DiffAN) & $0.41$ \var{0.10} & $\mathbf{0.38}$ \var{0.12} \\
\bottomrule
\end{tabular}
\vspace{-3em}

\end{wraptable}


\paragraph{Results.} 
In both cases with and without selection bias, our method outperforms the baselines. \rev{Because Sachs essentially provide perturbation data and we aggregate all the data, the result should be interpreted with care. In addition, we provide results on separate sub-datasets in Appendix~\ref{sec:sachsexp}, with varying selection strengths and shapes}.
 Overall, these experiments consistently demonstrate that the proposed method achieves competitive performance compared with other baselines. 

\vspace{-2mm}
\section{Conclusion}
\vspace{-2mm}
In this paper, we have presented a novel approach to address the challenges of hard selection bias in causal discovery. Our method leverages the strengths of the score function, which is preserved under hard selection, allowing us to effectively identify causal relationships even in the presence of selection bias, while all baseline methods suffer under selection. We have demonstrated the effectiveness of our approach through experiments on both synthetic and real-world datasets.

\paragraph{Limitations and Future Work.} First, our method assumes hard selection bias, which may not always be the case in real-world scenarios. Future work could explore extensions of our method to handle soft (non-deterministic) selection bias. Additionally, although we provide the theoretical result that the score function is preserved under hard selection,  in practice, we found that it is hard to estimate when the truncation is severe. Therefore, further research is needed to develop more robust estimation techniques for the score-matching based causal discovery.




\bibliographystyle{plainnat}
\bibliography{ref}

\appendix
\section{Examples and Definitions}\label{sec:example}

\subsection{Definitions}\label{sec:define}
\begin{definition}\label{def:markov} (Markov Condition (\citet{spirtes2000causation,pearl2009causality}))
    Given a DAG $\cG$ and distribution $P$ over the variable set $\X$, every variable $X$ in $\X$ is probabilistically independent of its non-descendants given its parents in $\cG$.
\end{definition}

\begin{definition}\label{def:faith} (Faithfulness Assumption (\citet{spirtes2000causation,pearl2009causality}))
    Given a DAG $\cG$ and distribution $P$ over the variable set $\X$, there are no independencies between variables that are not entailed by the Markov Condition.
\end{definition}

\subsection{Illustrative Example: Linear Chain \texorpdfstring{$X_1 \to X_2 \to X_3$}{X1 -> X2 -> X3}}\label{app:example}
\CHAIN*
 The noise gradients are
\[
    \nabla N_1 = (1, 0, 0), \qquad \nabla N_2 = (-a, 1, 0), \qquad \nabla N_3 = (0, -b, 1).
\]

\paragraph{Non-sink $X_1$  has $\rank(M_1^{\mathrm{cond}})=2$}: $\NB_1 = \{X_2\}$.
$H_1(x) = (s'_{n_1}(N_1) + a^2 s'_{n_2}(N_2),\; -a\,s'_{n_2}(N_2),\; 0)$. Conditioning on $X_2 = z_2$ gives $N_2 = z_2 - aN_1$:
\[
    H_1 \mid z_2 = \bigl(\underbrace{s'_{n_1}(N_1) + a^2 s'_{n_2}(z_2 - aN_1)}_{q_1(N_1)},\;\; \underbrace{-a\,s'_{n_2}(z_2 - aN_1)}_{q_2(N_1)},\;\; 0\bigr).
\]
In general, $q_1$ and $q_2$ are distinct nonlinear functions of $N_1$, i.e. $q_1(N_1)$ is not an affine transformation of $q_2(N_1)$ under non-Gaussian noise assumption ($s'_{n_1}(t) \neq \mathrm{const}$). Then the conditional matrix $M_1^{\mathrm{cond}}(z_2)$ has the form
\[
    M_1^{\mathrm{cond}}(z_2) = \begin{pmatrix} \Var(q_1) & \Cov(q_1, q_2) & 0 \\ \Cov(q_1, q_2) & \Var(q_2) & 0 \\ 0 & 0 & 0 \end{pmatrix}.
\]
i.e. the components of $H_1 | z_2$ span a 2-dimensional function space, thus $\rank(M_1^{\mathrm{cond}}) = 2$. 

\paragraph{Non-sink $X_2$ has $\mathrm{rank}(M_2^{\mathrm{cond}})=2$}: $\NB_2 = \{X_1, X_3\}$.
Conditioning on $(X_1, X_3) = (z_1, z_3)$ gives $N_3 = z_3 - b(az_1 + N_2) = c - wN_2$:
\[
    H_2 \mid (z_1, z_3) = s'_{n_2}(N_2) \cdot (-a, 1, 0) + b\,s'_{n_3}(c - wN_2) \cdot (0, b, -1).
\]
Components involve $s'_{n_2}(N_2)$ and $s'_{n_3}(c - wN_2)$ which are generally independent functions. Following the same argument as for $X_1$, $M_2^{\mathrm{cond}}(z_1, z_3)$ also has $\rank(M_2^{\mathrm{cond}}) = 2$ .

\paragraph{Sink $X_3$ has $\mathrm{rank}(M_3^{\mathrm{cond}})=1$}: $\NB_3 = \{X_2\}$.
$H_3(x) = s'_{n_3}(N_3) \cdot (0, -b, 1)$. Therefore:
\[
    M_3^{\mathrm{cond}}(z_2) = \Var(s'_{n_3}) \begin{pmatrix} 0 & 0 & 0 \\ 0 & b^2 & -b \\ 0 & -b & 1 \end{pmatrix}.
\]
\[\text{so,  } \rank(M_3^{\mathrm{cond}}) = 1\]

Therefore, the sink $X_3$ is uniquely identified as the only node with $\rank(M_i^{\mathrm{cond}}) = 1, i\in \{1, 2, 3\}$.

\section{Proofs}\label{sec:proofs}

\subsection{Preservation of Skeleton Discovery}\label{sec:proofs:edge}
\begin{proposition}\label{prop:edge-test-trunc}
        Let $\X$ be generated by a restricted ANM~\eqref{eq:anm} with mutually independent noise terms and let $S=\mathbbm{1}(x\in \cB)$ denote a deterministic selection  for a truncation region $\cB$. For any $i, j \in [d]$,
        \begin{equation}
            \partial^2_{X_i X_j} \log p(x{\color{blue}{\mid S=1}}) = 0 \Longleftrightarrow X_i \ind X_j \mid \X_{\setminus \{i, j\}} \text{in } \cG.
        \end{equation}
\end{proposition}
The proof of Proposition~\ref{prop:edge-test-trunc} directly follows of the preservation of score function under hard selection (truncation).

\subsection{Proofs of Conditional Hessian Rank Characterization}\label{sec:proofs:chess}
\begin{propbox}
\RANKCONDLIN*    
\end{propbox}

\begin{proof}
Let $\mathbf{X} = (X_1, \ldots, X_d)$ be generated by a linear additive noise model
\[
  X_i \;=\; \sum_{p \in X_{\PA_i}} a_{ip}\,X_{p} \;+\; N_i, \qquad i = 1, \ldots, d,
\]
where $N_1, \ldots, N_d$ are mutually independent, non-Gaussian, with strictly
positive smooth densities $p_{N_k}$. Let
$s_k(n) := (\log p_{N_k})'(n)$ and $s'_k(n) := (\log p_{N_k})''(n)$, and assume
$\Var(s'_k(N_k)) > 0$ for all $k$ (equivalent to non-Gaussianity).

\paragraph{Compute Hessian row $H_i$.}
Differentiating the score function with respect to $\mathbf{x}$ gives:
\begin{equation}
  H_i(\mathbf{x})
  \;=\; s'_i(N_i)\,\mathbf{v}_i \;-\; \sum_{k \in \CH_i} a_{k i}\,s'_k(N_k)\,\mathbf{v}_k,
  \label{eq:Hj_pf}
\end{equation}
where we define
$\mathbf{v}_k \coloneqq \nabla_{\mathbf{x}} N_k = \mathbf{e}_k - \sum_{\ell \in \PA_k} a_{k\ell}\mathbf{e}_{\ell}$
(a fixed vector in $\bbR^d$ because of linearity).
Thus $H_i$ depends on exactly the noise indices $\{i\} \cup \CH_i$.

By rewriting the coefficients as $c_{ij} := 1$ if $j = i$, and $c_{ij} := -a_{ji}$ if $j \in \CH_i$, we have:
$$H_i(x)=\sum_{j\in\mathcal Ni}c_{ij} s'_j(N_j)\,v_j=\tilde V_i g_i(x),$$

where $\tilde V_i=\left[c_{ij}v_j\right]_{j\in\mathcal{N}i}\in\mathbb{R}^{d\times|\mathcal{N}i|}, g_i=\big(s_j'(N_j)\big)_{j\in\mathcal{N}i}$ and $\tilde V_i$ is deterministic, constant in $x$, independent of $\mathcal{B}$. 

So the target matrix can be expressed as:
\begin{equation}
M_{i}^{\mathrm{cond}}(z)=\Var\left(H_i\mid X_{NB_i}=z\right)=\tilde V_i\Sigma_i(z)\tilde V_i^{\top},\label{eq:Mcond_final}
\end{equation}
where $\Sigma_i(z):=\operatorname{Cov}\big(g_i\mid X_{\mathrm{NB}i}=z\big)\in\mathbb R^{|\mathcal N_i|\times|\mathcal N_i|}.$

For the selected data, it is 
\begin{equation}
M_{i, \mathrm{sel}}^{\mathrm{cond}}(z)=\tilde V_i\Sigma_i^{\mathrm{sel}}(z)\tilde V_i^{\top},\label{eq:Mcond_final_sel}
\end{equation}
and 
$\Sigma_i^{\mathrm{sel}}(z):=\operatorname{Cov}\big(g_i\mid X_{\mathrm{NB}i}=z,\ {\color{blue}{S=1}}\big)\in\mathbb R^{|\mathcal N_i|\times|\mathcal N_i|}.$

Note that in ~\eqref{eq:Mcond_final_sel}, selection acts only inside $\Sigma_i^{\mathrm{sel}}(z)$, and the eigenvalues of $M_{i, \mathrm{sel}}^{\mathrm{cond}}(z)$ do differ from the untruncated ones, but it does not influence the rank. 




So $\Tilde{V}_i^\top$ has full column rank in general, so  $\rank M_i^\mathrm{cond}(z) = \rank \Sigma_i(z)=|\mathcal{N}_i|$.

\end{proof}

\begin{propbox}
    \RANKCONDL*
\end{propbox}

\begin{proof}
\textbf{Part (a): Sink}
Since $X_l$ is a sink, $H_l({x}) = s'_{n_l}(N_l)\cdot{v}_l$ where
${v}_l = \nabla N_l$ is a fixed vector under the linear model. Therefore:
\[
  M_l^{\mathrm{cond}}(z)
  = \mathrm{Var}(s'_{n_l}(N_l) \mid X_{\mathrm{NB}_l} = z)\cdot {v}_l{v}_l^\top,
\]
which is a scalar multiple of the \emph{same} rank-1 matrix ${v}_l{v}_l^\top$
for every $z$. Taking the expectation over $z$ and applying the law of total variance:
\[
  M_l^{\mathrm{cond}}
  = \mathbb{E}_z\bigl[\mathrm{Var}(s'_{n_l}(N_l)\mid X_{\mathrm{NB}_l}=z)\bigr]
    \cdot{v}_l{v}_l^\top
  = \mathrm{Var}(s'_{n_l}(N_l))\cdot{v}_l{v}_l^\top.
\]
Since $\mathrm{Var}(s'_{n_l}(N_l)) > 0$ by non-Gaussianity,
$\rank(M_l^{\mathrm{cond}}) = 1$.

\end{proof}

\begin{propbox}
\begin{restatable}[Conditional Rank Test for Sink Identification (Nonlinear)]{proposition}{RANKCOND}\label{prop:rank-cond-nL}
    Let $\X$ be generated by a restricted \emph{nonlinear} ANM~\eqref{eq:anm} with independent noise terms. Then:
    \begin{enumerate}[label=(\alph*),nosep]
        \item \textbf{Sink.} If $X_l$ is a sink (i.e., $\CH_l = \emptyset$), then:
        \begin{equation}\label{eq:sink_rank}
            \rank \!\bigl(M_l^{\mathrm{cond}}\bigr) = \begin{cases}1, &  \text{if } N_i \text{ is  non-Gaussian},\forall i\\
            0, &  \text{if } N_i \text{ is  Gaussian},\forall i\end{cases}.
        \end{equation}
        \item \textbf{Non-sink.} If $X_i$ is not a sink (i.e., $|\CH_i| \geq 1$), then: 
        \begin{equation}\label{eq:nonsink_rank}
            \rank \!\bigl(M_i^{\mathrm{cond}}\bigr) \geq  \begin{cases}
			2, & \text{if } N_i \text{ is  non-Gaussian},\forall i\\
            1, & \text{if } N_i \text{ is  Gaussian},\forall i
		 \end{cases}
        \end{equation}
    \end{enumerate}
    Therefore, a node $X_i$ is a sink if and only if $\rank(M_i^{\mathrm{cond}}) = 1$  when the noise is non-Gaussian, and $\rank(M_i^{\mathrm{cond}}) =0$ when the noise is Gaussian. 
\end{restatable}
\end{propbox}

\begin{proof}
\textbf{Part (a): Sink.}
Let $X_l$ be a sink, so $X_{\CH_l} = \emptyset$ and $X_{\NB_l} = X_{\PA_l}$. By ~\eqref{eq:Hj-decomp}, the Hessian row reduces to
\begin{equation}\label{eq:sink_H}
    H_l(x) = s'_{n_l}(N_l) \cdot \nabla N_l, \quad \text{where } \nabla N_l = e_l - \nabla_x f_l(x_{\PA_l}).
\end{equation}
There are two terms: $s_l'(N_l) $ and $\nabla N_l$ in this equation.  For $\nabla N_l$, conditioning on $X_{\NB_l} = X_{\PA_l} = z$ fixes $\nabla N_l$ to a constant vector
\[
    c_l(z) := e_l - \nabla_x f_l(z),
\]
where $e_i$ is the $i$-th standard basis vector. 
The exogenous noise $N_l$ remains free (not determined by the conditioning). Thus the only remaining random quantity is the scalar $s'_{n_l}(N_l)$, and
\[
    H_l \mid X_{\PA_l} \;=\; s'_{n_l}(N_l) \cdot c_l(z). \quad (X_{\PA_l}=z)
\]
The conditional covariance is
\[
    M_l^{\mathrm{cond}}(z) = \Var\!\bigl(s'_{n_l}(N_l)\bigr) \cdot c_l(z)\, c_l(z)^\top.
\]
When noise is non-Gaussian, $s'_{n_l}$ is non-constant, so $\Var(s'_{n_l}(N_l)) > 0$, and $c_l(z) \neq 0$ (as its $l$-th entry is 1). Therefore $\rank(M_l^{\mathrm{cond}}(z)) = 1$.

When noise is Gaussian, $s'_{n_l}(N_l) = -1/ \sigma^2$ is constant, so $\Var(s'_{n_l}(N_l)) = 0$ and $\rank(M_l^{\mathrm{cond}}(z)) = 0$.

\medskip

\textbf{Part (b): Non-sink.}
Let $X_i$ be a non-sink with at least one child $X_k \in \CH_i$. We condition on $X_{\NB_i} = z$, which partitions as $z = (z_{\PA_i}, z_{\CH_i})$.

Remind that the Hessian row for a non-sink $i$ is computed as:
\begin{equation}\label{eqn:Hj}
    H_i(x) = s_i'(N_i) \cdot \nabla N_i \;-\; \sum_{k \in \CH_i} \Big[\underbrace{s_k'(N_k) \cdot \frac{\partial f_k}{\partial x_i} \cdot \nabla N_k}_{\text{child noise direction}} + \underbrace{s_k(N_k) \cdot \nabla\!\left(\frac{\partial f_k}{\partial x_i}\right)}_{\text{mechanism curvature}}\Big],
\end{equation}

\emph{First term in ~\eqref{eqn:Hj}}: 
conditioning on $X_{\PA_i} = z_{\PA_i}$ fixes
\[
    \nabla N_i = e_i - \nabla_x f_i(z_{\PA_i}) =: c_i \quad (\text{constant vector} ),
\]
The noise $N_i = X_i - f_i(z_{\PA_i})$ remains free.

\emph{Child terms in ~\eqref{eqn:Hj}}: 
For child $X_k$, the structural equation reads $X_k = f_k(X_i, X_{\PA_k \setminus \{i\}}) + N_k$. Under the conditioning $X_k = z_k$ and $X_{\PA_k \setminus \{i\}} = z_{\PA_k \setminus \{i\}}$, the child noise becomes a function of $N_i$:
\begin{equation}\label{eq:child_noise}
    N_k = z_k - f_k\!\bigl(f_i(z_{\PA_i}) + N_i,\; z_{\PA_k \setminus \{i\}}\bigr) \coloneqq q_k(N_i;\, z).
\end{equation}

\emph{Overall},
substituting into~\eqref{eqn:Hj} and using that all $\nabla N_k$, $\partial f_k / \partial x_i$, and $\nabla(\partial f_k / \partial x_i)$ are deterministic functions once $z$ is fixed, the conditional Hessian row has the form
\begin{equation}\label{eq:cond_H}
    H_i \mid X_{\NB_i} = z \;=\; {\color{blue} {s'_i(N_i)}} \cdot c_i - \sum_{k \in \CH_i} {\color{purple}s'_k\!\bigl(q_k(N_i; z)\bigr)} \cdot \alpha_k(z) \cdot d_k(z) + R(N_i; z),
\end{equation}
where $\partial f_k / \partial x_i \coloneqq \alpha_k(z), \nabla N_k\coloneqq d_k(z)$ are deterministic scalars/vectors and $R(N_i; z)$ collects the mechanism curvature terms (also deterministic functions of $N_i$ through $q_k$).

\emph{Finally, consider the number of independent functions.}
The conditional Hessian row ~\eqref{eq:cond_H} is a vector-valued function of the single scalar $N_i$. Its components include at least two functionally distinct terms:
\begin{itemize}[nosep]
    \item ${\color{blue} {s'_i(N_i)}}$, arising from the own-noise contribution, and
    \item  ${\color{purple}s'_k\!\bigl(q_k(N_i; z)\bigr)}$, arising from the $k$-th child noise.
\end{itemize}

For the conditional rank to equal 1, every component of $H_i | z$ must be an affine function of a single ``base function'' $\psi(N_i)$. In particular, restricting to the entries at positions $i$ and $k$, this requires the existence of constants $\alpha, \beta$ such that
\begin{equation}\label{eq:affine}
    s'_i(t) = \alpha \cdot s'_k\!\bigl(q_k(t; z)\bigr) + \beta \quad \text{for all } t \in \mathrm{supp}(N_i).
\end{equation}
Since $q_k(t; z) = z_k - f_k(f_i(z_{\PA_i}) + t, z_{\PA_k \setminus \{i\}})$, this is a smooth invertible (generically) reparametrization of $t$. Equation~\eqref{eq:affine} then asserts that $s'_i$ and $s'_k \circ q_k(\cdot; z)$ are affinely related.

\rev{For non-Gaussian noise, under the restricted ANM identifiability condition~\citep{peters2014causal}, the noise score functions $s'_i$ and $s'_k$ are non-constant nonlinear functions of their arguments}. The composition $s'_k \circ q_k$ introduces an additional nonlinearity through the mechanism $f_k$. 
Formally, the set of $z$ for which the affine relation~\eqref{eq:affine} holds is contained in the zero set of the analytic function and has measure zero.

Therefore, for Lebesgue-almost every $z$, the components of $H_i | z$ span at least a 2-dimensional function space, and $\rank(M_i^{\mathrm{cond}}(z)) \geq 2$.

For Gaussian noise, $s'_i$ and $s'_k$ are constant, so  $H_i | z$ is composed of a linear combination of the base functions  $\psi(N_i)$, so  $\rank(M_i^{\mathrm{cond}}(z)) = 1$.
\end{proof}

\subsection{Proofs of Unconditional Hessian Rank Test}\label{sec:proofs:uncond}

The following method introduces the rank test of unconditional Hessian outer product matrix $M_i$ for sink identification, but no conditioning on the neighbors means that we are not blocking the contribution of noise terms from the upstream (parents) of $X_i$, thus the rank of $M_i$ is determined by the number of parents and children of $X_i$ instead of just the number of children. A particular challenge is to distinguish a sink node with $m$ parents from a source node with the same number of  $m$ children, as in Example~\ref{ex:uncondition}.
\begin{example}[Unconditional Rank Test Failure]\label{ex:uncondition}
        Consider a chain $X_1 \rightarrow X_2 \rightarrow X_3$, where the unconditional rank of $M_1$ is $2$, but the rank of $M_3$ is also $2$. Therefore, we cannot distinguish the sink node $X_3$ from the non-sink node $X_1$ by just looking at the rank of $M_i$.
    \end{example}

    More generally, the following result (Proposition~\ref{prop:rank}) characterizes the rank of $M_i$ for sink vs.\ non-sink nodes under both linear and nonlinear ANM.
    
\begin{restatable}[Unconditional Rank]{proposition}{RANK}\label{prop:rank}
Let $\X$ be generated by a restricted ANM~\eqref{eq:anm} with mutually independent non-Gaussian noise terms (i.e., $s_i'$ is non-constant for each $i$).
\begin{enumerate}[label=(\alph*)]
    \item \textbf{sink:} If $X_l$ is a sink with $|\PA_l|$ parents, then for linear mechanisms, $\rank(M_l) = 1$. For nonlinear mechanisms, $\rank(M_l) \leq |\PA_l| + 1$.
    \item \textbf{Non-sink:} If $X_i$ has $|\PA_i|$ parents and $ |\CH_i| \geq 1$ children, then for linear mechanisms, $\rank(M_i) \geq 2$. For nonlinear mechanisms, $\rank(M_i) \leq |\PA_i| + |\CH_i| + 1$ and generically $\rank(M_i) > |\PA_i| + 1$.
\end{enumerate}
If $X_l$ is a sink node, it is identifiable if and only if for every non-sink node $X_i$,
$|\PA_i| + |\CH_i| - |\PA_l| \geq 1$, 
(i.e. there exists a rank gap between a sink node and all non-sink nodes). 
\end{restatable}

\begin{proof}[Proof sketch]
For a sink $X_l$: the decomposition \eqref{eq:Hj-decomp} reduces to $H_l(x) = s_l'(N_l) \cdot \nabla N_l$. If $f_l$ is linear, then $\nabla N_l$ is a constant vector $c_l$, and every $H_l(x)$ is parallel to $c_l$. Thus $M_l = \bbE[s_l'^2] \cdot c_l c_l^\top$, which has rank $1$. If $f_l$ is nonlinear, $\nabla N_l$ varies with $x^{\PA_l}$ and spans at most $m + 1$ dimensions (the entries corresponding to $\{l\} \cup \PA_l$).

For a non-sink $X_i$ with child $X_k$: $H_i$ contains an additional term with coefficient $s_k'(N_k)$ (independent of $s_i'(N_i)$ since $N_i \ind N_k$) multiplying the direction $\nabla N_k$. Since $\nabla N_k$ has a non-zero entry at position $k \notin \{i\} \cup \PA_i$, it is linearly independent of $\nabla N_i$. The two independent random coefficients times two linearly independent directions produce a matrix of rank $\geq 2$ (linear case), or of rank at least $\rank(M_l) + 1$ (nonlinear case).

Under non-Gaussianity, $s_i'$ is non-constant, preventing the random coefficients from collapsing to scalars (which would happen for Gaussian noise where $s'(N) = -1/\sigma^2$ is constant).

Under Gaussianity but with non-linear mechanisms, $s_i'$ is constant, but the mechanism curvature term $\nabla\!\left(\frac{\partial f_k}{\partial x_j}\right)$ is generally non-zero and linearly independent of $\nabla N_i$, so the rank gap still exists.
\end{proof}

\begin{table}[h]
\centering
\caption{Rank of $M_j$ for sink vs.\ non-sink nodes across different ANM settings.}
\label{tab:rank}
\renewcommand{\arraystretch}{1.3}
\begin{tabular}{@{}lccl@{}}
\toprule
\textbf{Setting} & \textbf{sink rank} & \textbf{Non-sink rank} & \textbf{Gap} \\
\midrule
\textcolor{black}{Linear + Gaussian} & \textcolor{black}{$1$} & \textcolor{black}{$1$} & \textcolor{black}{$0$ (not identifiable)} \\

Linear + non-Gaussian & $1$ & $\geq 1 + |\CH_l|$ & $|\CH_j|\geq 1$ \\
Nonlinear + Gaussian & $\leq |\PA_l| + 1$ & $\leq |\PA_l| + |\CH_l| + 1$ & $|\PA_j| + |\CH_j| - |\PA_l|$ \\
Nonlinear + non-Gaussian & $\leq |\PA_l| + 1$ & $\leq |\PA_l| + |\CH_l| + 1$ & $|\PA_j| + |\CH_j| - |\PA_l|$ \\
\bottomrule
\end{tabular}
\end{table}

\begin{remark}[Comparing Conditional and Unconditional Rank Test]
The unconditional rank test in Proposition~\ref{prop:rank} imposes stronger conditions on the graph structure, since it requires that $|\PA_j| + |\CH_j| - |\PA_l| \geq 1$ for all non-sink nodes $X_i$ for sink-node identification. A tree-like structure (where a non-sink node has more than one child and a sink node has at most one parent) satisfies this condition, but more general graphs may not. In contrast, the conditional rank test in Proposition~\ref{prop:rank-cond-L} and~\ref{prop:rank-cond-nL} does not require such a degree constraint, since conditioning on the neighbors blocks the contribution of the parents, thus only counting the number of noise terms that directly feed into the score of $X_i$.
\end{remark}

\subsection{Proofs of Consistent Estimator of $M_i^\text{cond}$}\label{sec:app:estim}
\paragraph{Conditional Variance Matrix Estimation}
     In Algorithm~\ref{alg:estim}, we directly estimate $M_i^\text{cond}(z)$ via a local kernel-weighted covariance estimator. In practice, when the ANM is linear non-Gaussian or nonlinear Gaussian, we find that the residual-based estimator is more robust to estimation errors in the Hessian, as a direct estimator of $M_i^\text{cond}$ relies on selecting the number of conditional neighbors $x_{\NB_i}$ and the kernel bandwidth, thus suffers more from estimation errors.

     Specifically, we define the test statistic as $\widehat{M}_i^{\text{cond}} := \frac{1}{n}\sum_{t=1}^n r_i(x^{(t)}) r_i(x^{(t)})^\top$ as a proxy, where $r_i(x) := H_i(x) - \hat{m}_i(x_{\text{NB}_i})$ is the residual of using  $\hat{m}_i(\cdot)$ to estimate $H_i$, i.e. $\hat{m}_i(\cdot)$ is a regressor of $H_i$ on $X_{\text{NB}_i}$ approximating $m_i \coloneqq \bbE[H_i \mid X_{\text{NB}_i}]$. 

\begin{restatable}[Consistent Estimator of $M_i^\text{cond}$]{proposition}{ESTIM}\label{prop:estimM}
    Under a linear non-Gaussian or nonlinear Gaussian ANM, $\widehat{M_i}^{\mathrm{cond}}$ is a consistent estimator of $M_i^\text{cond}$ if $\hat{m}_i$ is a consistent estimator of $m_i$.
\end{restatable}

\begin{proof}
The proof of Proposition~\ref{prop:estimM} relies on the law of total variance: $\bbE M_i^{\text{cond}} = \text{Var}(H_i) - \text{Var}(m_i(X_{\text{NB}_i})).$ We first show that $M_i^\mathrm{cond}  = \bbE[r_i r_i^\top]$, and then by the continuous mapping theorem, If $\hat{m}_i \to m_i$ in $L^2$ as $n \to \infty$, then $\widehat{M}_i^{\text{cond}} \to M_i^{\text{cond}}$.
    \[
    \begin{aligned}
        \bbE[r_i r_i^\top] 
            &= \bbE(H_i(x) - \hat{m}_i)(H_i(x) - \hat{m}_i)^\top\\
            &= \bbE[H_i H_i^\top] - \bbE [\hat{m}_i H_i^\top] - \bbE [H_i \hat{m}_i^\top] + \bbE \hat{m}_i\hat{m}_i^\top
    \end{aligned}
    \]

    By the towering property, we can expand $H_i =  \bbE[H_i \mid X_{\text{NB}_i}] = m_i$, so
    \[
    \begin{aligned}
           \bbE[r_i r_i^\top] 
            &= \bbE[H_i H_i^\top] - \bbE [\bbE[H_i \mid X_{\text{NB}_i}] H_i^\top] - \bbE [H_i  \bbE[H_i \mid X_{\text{NB}_i}]^\top] + \bbE \hat{m}_i\hat{m}_i^\top\\
            &= \bbE[H_i H_i^\top] -\bbE [\hat{m}_i\hat{m}_i^\top] - \bbE [\hat{m}_i\hat{m}_i^\top]+ \bbE [\hat{m}_i\hat{m}_i^\top] \\
            &= \bbE[H_i H_i^\top] - \bbE [\hat{m}_i\hat{m}_i^\top].
    \end{aligned}
    \]
    On the other hand, 
    \[
    \begin{aligned}
        \bbE M_i^\text{cond} &= \text{Var}(H_i) - \text{Var}(m_i(X_{\text{NB}_i}))\\
        &= \bbE [H_i H_i^\top] - \bbE [\hat{m}_i\hat{m}_i^\top] - \bbE [\hat{m}_i\hat{m}_i^\top] + \bbE [\hat{m}_i\hat{m}_i^\top]\\
        &= \bbE[H_i H_i^\top] - \bbE [\hat{m}_i\hat{m}_i^\top]\\
        &= \bbE[r_i r_i^\top]
    \end{aligned}
    \]

    When the ANM is linear non-Gaussian or nonlinear Gaussian, the local and averaged conditional matrix coincide, i.e. $M_i^\mathrm{cond}(z) = \bbE_z M_i^\mathrm{cond}$ for generic z. As is in the proof of Proposition~\ref{prop:rank}, for a sink node $l$, the geometrical factorization of the target matrix $M_l^{\mathrm{cond}}(z)=\operatorname{Var}\left(s_{n_l}^{\prime}\left(N_l\right)\right) \cdot c_l(z) c_l(z)^{\top}$, where $c_l(z)=e_k -\nabla_x f_l(z)$. In the linear case, $c_l(z)$ is a constant vector that does not rotate with $z$. Therefore the aggregation $\mathbb{E}_zM_l^{\mathrm{cond}}(z)$ is an aggregation of rank $1$ matrices that has the same eigenvector,  therefore its value is preserved after pooling for linear ANM.     
    In the nonlinear-Gaussian case, $\operatorname{Var}\left(s_{n_l}^{\prime}\right)=0$ thus gives 0 rank of the matrix, and the rotation of $c_l(z)$ also does not matter.

\end{proof}
\section{Related Works}\label{sec:related}
\paragraph{Causal Discovery and Inference under Selection Bias}
There are limited works on causal discovery under selection bias, and they either focus on local causal structure learning~\citep{versteegLocalConstraintBasedCausal2022a}, or assuming access to interventional data~\citep{daiWhenSelectionMeets2025,luo2025gene}, or assuming sequential data~\citep{zhengDetectingIdentifyingSelection2024b}, or performing discovery at latent variable level~\citep{daiLatentVariableCausal2025} with linear Gaussian assumption. In contrast, we focus on learning the full causal structure among observed variables under selection bias, without requiring interventional or structural assumptions on the selection mechanism, and allowing for nonlinear and non-Gaussian mechanisms.

\paragraph{Causal Inference under Selection Bias}
Selection bias is a well-known problem in causal inference, and there are various methods to address it, such as  \citep{bareinboim2014recovering,bareinboim2012controlling,bareinboim2015recovering,correaIdentificationCausalEffects2019,bareinboim2022recovering}, which characterizes conditions under which causal effects remain identifiable despite selection bias. 
Subsequent works extended this approach by developing testable implications of selection mechanisms \citep{correaIdentificationCausalEffects2019} and providing adjustment criteria. Additionally, \cite{van2014constructing, jaber2019causal} developed causal effect identification results under Markov equivalence classes (MEC) which takes into account the potential presence of latent selections. Selection bias has been studied extensively in economics \cite{heckman1977sample} and epidemiology \cite{robins2000marginal}. Selection bias has been studied extensively in economics \cite{heckman1977sample} and epidemiology \cite{robins2000marginal}. However, these works primarily focus on \emph{correcting} the bias rather than causal structure learning, which is the focus of our work.



\section{Algorithm}\label{sec:algo}

\subsection{Main Algorithm}\label{sec:algo:main}
The overall algorithm for causal discovery via the Conditional Hessian Rank Test (CHess) is outlined in Algorithm~\ref{alg:main}. 
\begin{algorithm}[H]
\caption{Causal discovery via Conditional HessianRank Test (CHess)}\label{alg:main}
\begin{algorithmic}[1]
\Require Data $\{x^{(i)}\}_{i=1}^n$ (potentially  from truncated distribution $P(\X \mid S = 1)$)
\Ensure Causal DAG $\cG$
\State \textbf{Estimate Hessian:} Compute $\nabla^2 \log p(x^{(i)})$ at each data point (e.g., via Stein estimator or kernel score matching on the truncated data).
\State $\mathcal{\X} \leftarrow \{X_1, \ldots, X_d\}$
\While{$\mathcal{\X} \neq \emptyset$}
    \State \textbf{MEC recovery:} For all pairs $X_i, X_j \in \mathcal{\X}$, test $\frac{\partial^2 \log p\mid x}{\partial X_i \partial X_j} = 0$ for conditional independence~\citep{montagnaScalableCausalDiscovery2023a}. Recover skeleton  via Proposition~\ref{prop:edge-test-trunc}.
    \State \textbf{Sink identification:} For each $X_j \in \mathcal{\X}$, find its neighbor $\NB_j$ and estimate its conditional hessian matrix rank $r_j$ via Algorithm~\ref{alg:estim}.
    \State $X_l \leftarrow \arg\min_{j} r_j$ \hfill \Comment{Smallest estimated rank = sink}
    \State \textbf{Orient edges:} Orient all edges adjacent to $X_l$ as incoming.
    \State $\mathcal{\X} \leftarrow \mathcal{\X} \setminus \{X_l\}$; update Hessian estimates for the reduced variable set.
    \State \textbf{DAS pruning}: a final pruning step that is robust to hard selection~\citep{montagnaScalableCausalDiscovery2023a}. 
\EndWhile
\State \Return $\cG$
\end{algorithmic}
\end{algorithm}

\subsection{Algorithm for Conditional Hessian Matrix Estimation}\label{sec:algo:matrix}
The algorithms used for estimating $M_i^\text{cond}$ and sink node identification. We present two alternative algorithms for estimating the conditional Hessian matrix rank: (1) Algorithm~\ref{alg:estim} is a conditional rank estimator based on local averaging, used for nonlinear non-Gaussian ANMs and (2) Algorithm~\ref{alg:estim2} is a residual rank estimator based on regression, used for linear non-Gaussian ANMs and nonlinear Gaussian ANMs. 

\begin{algorithm}[htbp]
\caption{Conditional Rank Estimation: Kernel-Weighted Local Covariance}\label{alg:estim}
\begin{algorithmic}[1]
\Require Data $\{x^{(i)}\}_{i=1}^n$; neighborhood sets $\{\NB_j\}_{j=1}^d$ from MEC recovery; number of reference points $L$; bandwidth $h$.
\Ensure Identified sink node $X_l$.

\For{each node $X_j \in V$}
    \State Select $L$ reference points $\{z^{(\ell)}\}_{\ell=1}^L$ in the support of ${\NB_j}$
    \Comment{e.g., $K$-means centroids}
    \For{$\ell = 1, \dots, L$}
        \State Compute weights: $w_i^{(\ell)} \leftarrow K_h(x^{(i)}_{\NB_j} - z^{(\ell)}), \quad K_h(u) = \exp\!\Bigl(-\frac{\|u\|^2}{2h^2}\Bigr),$ \quad $i = 1, \dots, n$
        \State Compute $\bar{H}_j^{(l)} = \frac{\sum_i w_i^{(l)}\, H_j(x^{(i)})}{\sum_i w_i^{(l)}}.$ 
        \State Compute $M_j^{\text{cond}}(z^{(l)}) = \frac{\sum_i w_i^{(l)}\, \bigl(H_j(x^{(i)}) - \bar{H}_j^{(l)}\bigr)\bigl(H_j(x^{(i)}) - \bar{H}_j^{(l)}\bigr)^\top}{\sum_i w_i^{(l)}}.$ 
        \State Compute eigenvalues $\hat{\lambda}_1^{(\ell)} \geq \hat{\lambda}_2^{(\ell)} \geq \cdots \geq \hat{\lambda}_d^{(\ell)}$
        \State $r_j^{(\ell)} \leftarrow (\sum_k \hat{\lambda}_k^{(\ell)})^2 / \sum_k (\hat{\lambda}_k^{(\ell)})^2$
        \Comment{effective rank at $z^{(\ell)}$}
    \EndFor
    \State $\hat{r}_j \leftarrow \mathrm{median}(r_j^{(1)}, \dots, r_j^{(L)})$
    \Comment{aggregate over reference points}
\EndFor
\State $X_l \leftarrow \arg\min_j \hat{r}_j$
\Comment{smallest conditional effective rank $=$ sink}
\State \Return $X_l$
\end{algorithmic}
\end{algorithm}

\begin{algorithm}[h]
\caption{Residual Rank Estimator}\label{alg:estim2}
\begin{algorithmic}[1]

\Require Data $\{x^{(i)}\}_{i=1}^n$; neighborhood sets $\{\NB_j\}_{j=1}^d$ from MEC recovery
\Ensure Identified sink node $X_l$

\For{each candidate node $X_j$}
    \For{$i = 1, \ldots, n$}
        \State Compute $H_j(x^{(i)}) \leftarrow \nabla_x\, \partial_{x_j} \log \hat{p}(x^{(i)}) \in \mathbb{R}^d$
        \Comment {estimate Hessian rows}
    \EndFor

    \State Fit $\hat{m}_j \leftarrow \arg\min_{f} \frac{1}{n}\sum_{i=1}^n \left\| H_j(x^{(i)}) - f(x_{\NB_j}^{(i)}) \right\|^2$
    \Comment {estimate conditional mean via regression}

    \For{$i = 1, \ldots, n$}
    \State $e_j^{(i)} \leftarrow H_j(x^{(i)}) - \hat{m}_j(x_{\NB_j}^{(i)}) \in \mathbb{R}^d$
    \Comment {compute residuals}
    \EndFor

    \State $\widehat{M}_j^{\mathrm{cond}} \leftarrow \dfrac{1}{n} \sum_{i=1}^n e_j^{(i)}\, \bigl(e_j^{(i)}\bigr)^\top \in \mathbb{R}^{d \times d}$
    \Comment {form conditional outer-product matrix}

    \State Compute eigenvalues $\hat{\lambda}_1 \geq \hat{\lambda}_2 \geq \cdots \geq \hat{\lambda}_d \geq 0$ of $\widehat{M}_j^{\mathrm{cond}}$
    \State $\widehat{r}_j \leftarrow \mathrm{Rank}\!\left(\widehat{M}_j^{\mathrm{cond}}\right)$
    \Comment {effective rank}

\EndFor

\State \Return $X_l \leftarrow \arg\min_{j}\, \widehat{r}_j$

\end{algorithmic}
\end{algorithm}
\section{More Details on the Experiments}\label{sec:exp_details}
\subsection{Datasets}
\paragraph{Synthetic Dataset} The binary adjacency $\mathbf{B}_\text{bin}$ is sampled by drawing
$\lfloor d \cdot \text{degree} / 2 \rfloor$ edges uniformly over all possible indices.  For \textbf{linear SEM}, we generate dge weights $\mathbf{B}_{ij}$ uniformly from
$[-0.8,-0.2]\cup[0.2,0.8]$, and per-node noise variances are drawn
i.i.d.\ from $\mathcal{U}(1,3)$. For each configuration we sample $n=1000$ observations and average over $20$ seeds. And the datas are generated as: $X_i = \mathbf{B}_{\cdot i}^\top \mathbf{X} + \epsilon_i,$ with $\epsilon_i$ drawn from a Gumbel or log-normal distribution. 
For \textbf{nonlinear} SEM, we use a one-layer MLP with $64$ hidden units and a sigmoid activation: $f_i(\mathbf{x}) = \sigma(\mathbf{x}\,W_1)\,W_2$, where the entries of $W_1\in\mathbb{R}^{|\mathrm{Pa}(X_i)|\times 64}$ and $W_2\in\mathbb{R}^{64}$ are drawn i.i.d.\ from $\mathcal{U}(0.5, 2)$ with each entry's sign flipped independently with probability $1/2$.

\paragraph{Sachs~\citep{sachs2005causal} Dataset}
The Sachs data~\citep{sachs2005causal} is a single-cell biological dataset in order to estimate protein interactions. It consists of nine files, each with one interventional context and the same set of variables $ raf, mek, plc, pip2, pip3, erk,
akt, pka, pkc, p38, jnk$, representing the measurements of different biochemical concentrations. There are between $700$ and $900$ data points in each file, for a total of $7466$ data points. The data is continuous and very skewed, shifting far away from our assumptions of Gaussian distributions. Following~\citet{ramsey2018fask}, we apply log transformation to the data.  

\paragraph{DREAM4 Dataset} We consider the in-silico Size-$10$ sub-challenge of DREAM4\footnote{\url{https://gnw.sourceforge.net/resources/DREAM4\%20in\%20silico\%20challenge.pdf}}
\citep{pinna2010knockouts,shah2023causal}, which provides five
ground-truth gene regulatory networks of $d=10$ genes each. 

\subsection{Estimation Method}\label{sec:score-estim}
Diffusion model-based score estimation~\citep{sanchezDiffusionModelsCausal2023} sidesteps the boundary problem by learning the score $\nabla_x \log p_\sigma(x)$ via denoising score matching with noise level $\sigma$. We use diffusion probablistic models (DPMs) to learn to reverse a diffusion process, by starting with noise and recovery data distribution to be estimated. A DPM is trained with the denoising score matching objective:

$$\theta^* = \arg\min_\theta \, \bbE_{x_0, t, \epsilon}\!\left[\lambda(t)\|\epsilon_\theta(x_t, t) - \epsilon\|_2^2\right]$$

where $x_t = \sqrt{\alpha_t} x_0 + \sqrt{1-\alpha_t}\epsilon$, with $\epsilon \sim \mathcal{N}(0, I)$. The key result~\citep{song2019generative} is that the trained noise predictor $\epsilon_\theta$ approximates the score of the data distribution: $\epsilon_\theta(x, t) \approx \nabla_x \log p(x).$ Once $\epsilon_\theta$ is trained, the $(i,j)$-th entry of the Hessian is obtained by a single backpropagation pass through the network: $H_{i,j}(\log p(\mathbf{x})) \approx \nabla_{i,j} \epsilon_\theta(\mathbf{x}, t)$, i.e. differentiate the $i$-th output of $\epsilon_\theta$ with respect to the $j$-th input.

\subsection{Evaluation matrics} 
The output of PC~\citep{spirtes2000causation} and GES~\citep{chickering03_ges} are complete partially directed acyclic grap (CPDAGs), while the output of FCI~\citep{spirtesCausalInferencePresence2013a} is a partial acyclic graph (PAG). We convert the output of the baseline algorithms according to Table.~\ref{tab:metrics}.

\begin{table}[htbp]
    \centering
    \renewcommand{\arraystretch}{1.2}
    \begin{tabular}{cc|cc}
    \toprule
    \textbf{True} & \textbf{Estimated} & \textbf{Adjacency} & \textbf{Arrowhead} \\
    \midrule
    $a \leftarrow b$ & $a \leftarrow b$ & TP & TP, TN \\
     & $a \to b$ & TP & FP, FN \\
     & $a \text{ --- } b$ & TP & FN \\
     & $a \quad \ \ b$ & FN & FN \\
    \midrule
    $a \quad\ \  b$ & $a \leftarrow b$ & FP & FP \\
     & $a \to b$ & FP & FP \\
     & $a \text{ --- } b$ & FP & ${\backslash}$ \\
     & $a \quad\ \ b$ & TN &  $\backslash$ \\
    \bottomrule
    \end{tabular}
    \vspace{5pt}
    \caption{Metrics used for evaluation. Precision=TP/(TP+FP), Recall=TP/(TP+FN).}
    \label{tab:metrics}
\end{table}

\subsection{Hyperparameters}
We have the following hyperparameters for our method, which are set to the same values across all experiments unless otherwise specified:
\begin{enumerate}
    \item   MEC threshold: relative off-diagonal cutoff for the MEC skeleton recovery step. An off-diagonal entry of the RMS Hessian matrix is treated as a skeleton edge if it exceeds MEC threshold. Set to 0.1 in all experiments.
    \item Rank test: method for estimating the numerical rank of $M_j^{\mathrm{cond}}$. \verb|chi2| uses an eigenvalue-based $\chi^2$ test (EigRankTest) at significance level. Set to \verb|chi2| in all scripts.                                                    
    \item Rank test significance level: $\alpha$ for the $\chi^2$ rank test above. Set to 0.05 in all scripts.
    \item KRR ridge: the ridge regularization parameter $\lambda$ for the kernel ridge regressor used by residual estimator of $M_j^{\mathrm{cond}}$. Controls the bias–variance tradeoff of the residual score estimate. Set to $10^{-5}$ in all scripts.
    \item CAM pruning threshold: p-value cutoff for the final CAM-style edge pruning step. Each candidate edge is tested by a GAM regression; edges whose coefficient p-value exceeds this threshold are removed. Set to $0.001$ for DREAM4, $0.01$ for Sachs.
\end{enumerate}
\subsection{Baseline Implementations}
For PC, GES, FCI, GRaSP, LiNGAM, we use the implementations in the \texttt{causallearn} Python package~\citep{zheng2024causal} (\url{https://github.com/py-why/causal-learn}), with the default hyperparameters. For CAM, SCORE, DAS, and NoGAM, we use the implementations in \texttt{dodiscover} (\url{https://www.pywhy.org/dodiscover/dev/index.html}), with the default hyperparameters.
\section{More Experimental Results}\label{sec:expresults}
\subsection{All Baselines in Linear non-Gaussian ANM}
\paragraph{Results.} When hard selection bisa is not present, LiNGAM-based methods achieves the best performance, which is not a surprising result. However, when selection bias becomes severe, our conditional Hessian rank-based method starts to outperform. We also \underline{underlines} the second-best performance.
\begin{table}[htbp]
\centering
\caption{$d=5$, noise=lognormal: F1 Metrics}
\renewcommand{\arraystretch}{1.2}
\begin{tabular}{l|ccc|ccc}
\toprule
 & \multicolumn{3}{c|}{F1 (Adj) $\uparrow$} & \multicolumn{3}{c}{F1 (Arrow) $\uparrow$} \\
Method & $K=0$ & $K=1$ & $K=2$ & $K=0$ & $K=1$ & $K=2$ \\
\midrule
PC & $\mathbf{0.96}(0.00)$ & $0.76(0.02)$ & $0.70(0.02)$ & $\backslash$ & $\backslash$ & $\backslash$ \\
FCI & $\mathbf{0.96}(0.00)$ & $0.76(0.02)$ & $0.70(0.02)$ & $\backslash$ & $\backslash$ & $\backslash$ \\
Exact & $0.70(0.02)$ & $0.65(0.02)$ & $0.66(0.01)$ & $0.29(0.04)$ & $0.31(0.02)$ & $0.34(0.06)$ \\
GES & $0.96(0.00)$ & $\underline{0.78}(0.02)$ & $0.72(0.02)$ & $\backslash$ & $\backslash$ & $\backslash$ \\
GRaSP & $0.89(0.01)$ & $\mathbf{0.79}(0.03)$ & $0.71(0.02)$ & $\backslash$ & $\backslash$ & $\backslash$ \\
ICA-LiNGAM & $\underline{0.96}(0.00)$ & $0.74(0.02)$ & $0.72(0.01)$ & $\mathbf{0.95}(0.01)$ & $0.59(0.07)$ & $0.50(0.08)$ \\
Direct-LiNGAM & $\underline{0.96}(0.00)$ & $\mathbf{0.79}(0.02)$ & $\underline{0.74}(0.01)$ & $\mathbf{0.95}(0.01)$ & $\mathbf{0.76}(0.02)$ & $\underline{0.68}(0.04)$ \\
CAM & $0.78(0.01)$ & $0.71(0.01)$ & $0.72(0.01)$ & $0.12(0.03)$ & $0.07(0.01)$ & $0.13(0.02)$ \\
SCORE & $0.91(0.01)$ & $0.71(0.01)$ & $0.68(0.02)$ & $\underline{0.89}(0.01)$ & $0.53(0.04)$ & $0.47(0.06)$ \\
DAS & $0.91(0.01)$ & $0.71(0.01)$ & $0.68(0.02)$ & $\underline{0.89}(0.01)$ & $0.53(0.04)$ & $0.47(0.06)$ \\
NoGAM & $0.91(0.01)$ & $0.74(0.01)$ & $0.69(0.02)$ & $0.83(0.03)$ & $0.54(0.05)$ & $0.49(0.05)$ \\
DiffAN & $0.79(0.01)$ & $0.74(0.01)$ & $0.71(0.01)$ & $0.26(0.05)$ & $0.30(0.03)$ & $0.23(0.05)$ \\
CHess (Stein) & $0.94(0.01)$ & $0.76(0.01)$ & $0.73(0.01)$ & $\underline{0.91}(0.02)$ & $0.70(0.03)$ & $\mathbf{0.70}(0.02)$ \\
CHess (DiffAN) & $0.93(0.01)$ & $0.76(0.01)$ & $\mathbf{0.74}(0.02)$ & $0.88(0.03)$ & $\underline{0.72}(0.02)$ & $\mathbf{0.70}(0.03)$ \\
\bottomrule
\end{tabular}
\end{table}

\begin{table}[htbp]
\centering
\caption{$d=5$, noise=lognormal: SHD Metrics}
\renewcommand{\arraystretch}{1.2}
\begin{tabular}{l|ccc}
\toprule
 & \multicolumn{3}{c}{SHD $\downarrow$} \\
Method & $K=0$ & $K=1$ & $K=2$ \\
\midrule
PC & $2.81(1.85)$ & $4.81(3.26)$ & $5.11(2.25)$ \\
FCI & $4.48(0.92)$ & $5.96(0.85)$ & $6.33(2.07)$ \\
Exact & $6.11(2.84)$ & $6.33(2.59)$ & $6.22(4.02)$ \\
GES & $2.78(1.58)$ & $4.63(3.79)$ & $5.70(3.54)$ \\
GRaSP & $3.44(1.80)$ & $4.00(2.74)$ & $5.15(2.13)$ \\
ICA-LiNGAM & $\mathbf{0.41}(0.39)$ & $4.00(4.59)$ & $5.04(7.07)$ \\
Direct-LiNGAM & $\mathbf{0.41}(0.39)$ & $\mathbf{2.63}(2.16)$ & $3.63(3.20)$ \\
CAM & $7.15(3.83)$ & $8.30(1.34)$ & $7.43(1.65)$ \\
SCORE & $1.11(1.14)$ & $4.43(3.25)$ & $5.03(5.37)$ \\
DAS & $1.11(1.14)$ & $4.43(3.25)$ & $5.03(5.37)$ \\
NoGAM & $1.44(2.32)$ & $4.33(4.22)$ & $5.03(3.97)$ \\
DiffAN & $6.04(4.63)$ & $6.26(3.08)$ & $6.78(3.95)$ \\
CHess (Stein) & $\underline{0.81}(1.19)$ & $\underline{3.17}(2.34)$ & $\underline{3.53}(1.92)$ \\
CHess (DiffAN) & $1.11(2.32)$ & $\underline{3.03}(2.10)$ & $\mathbf{3.47}(2.85)$ \\
\bottomrule
\end{tabular}
\end{table}

\begin{table}[htbp]
\centering
\caption{$d=10$, noise=lognormal: F1 Metrics}
\renewcommand{\arraystretch}{1.2}
\begin{tabular}{l|ccc|ccc}
\toprule
 & \multicolumn{3}{c|}{F1 (Adj) $\uparrow$} & \multicolumn{3}{c}{F1 (Arrow) $\uparrow$} \\
Method & $K=0$ & $K=1$ & $K=2$ & $K=0$ & $K=1$ & $K=2$ \\
\midrule
PC & $0.90(0.01)$ & $0.76(0.01)$ & $0.73(0.01)$ & $0.59(0.04)$ & $\backslash$ & $\backslash$ \\
FCI & $0.90(0.01)$ & $0.76(0.01)$ & $0.73(0.01)$ & $0.55(0.03)$ & $\backslash$ & $\backslash$ \\
Exact & $0.53(0.00)$ & $0.50(0.00)$ & $0.46(0.00)$ & $0.23(0.01)$ & $0.25(0.01)$ & $0.21(0.01)$ \\
GES & $0.89(0.00)$ & $0.70(0.01)$ & $0.64(0.01)$ & $0.58(0.04)$ & $0.49(0.02)$ & $0.44(0.02)$ \\
GRaSP & $0.88(0.01)$ & $\mathbf{0.80}(0.01)$ & $\mathbf{0.75}(0.02)$ & $0.55(0.05)$ & $\backslash$ & $\backslash$ \\
ICA-LiNGAM & $\underline{0.93}(0.00)$ & $0.66(0.01)$ & $0.58(0.01)$ & $\underline{0.92}(0.00)$ & $0.51(0.05)$ & $0.40(0.04)$ \\
Direct-LiNGAM & $\mathbf{0.93}(0.00)$ & $0.72(0.01)$ & $0.62(0.01)$ & $\mathbf{0.93}(0.00)$ & $\mathbf{0.71}(0.01)$ & $\mathbf{0.60}(0.02)$ \\
CAM & $0.72(0.01)$ & $0.61(0.00)$ & $0.57(0.01)$ & $0.09(0.01)$ & $0.09(0.01)$ & $0.09(0.01)$ \\
SCORE & $0.80(0.01)$ & $0.65(0.01)$ & $0.58(0.01)$ & $0.75(0.03)$ & $0.52(0.02)$ & $0.40(0.03)$ \\
DAS & $0.82(0.01)$ & $0.66(0.01)$ & $0.60(0.02)$ & $\underline{0.77}(0.02)$ & $0.56(0.02)$ & $0.46(0.03)$ \\
NoGAM & $0.81(0.01)$ & $0.63(0.01)$ & $0.58(0.01)$ & $0.67(0.04)$ & $0.46(0.03)$ & $0.42(0.03)$ \\
DiffAN & $0.70(0.01)$ & $0.61(0.01)$ & $0.57(0.01)$ & $0.22(0.02)$ & $0.21(0.02)$ & $0.17(0.01)$ \\
CHess (Stein) & $0.80(0.01)$ & $0.65(0.01)$ & $0.58(0.01)$ & $0.68(0.03)$ & $\underline{0.57}(0.02)$ & $0.48(0.02)$ \\
CHess (DiffAN) & $0.81(0.01)$ & $0.65(0.01)$ & $0.58(0.01)$ & $0.70(0.02)$ & $\underline{0.58}(0.02)$ & $\underline{0.51}(0.01)$ \\
\bottomrule
\end{tabular}
\end{table}

\begin{table}[htbp]
\centering
\caption{$d=10$, noise=lognormal: SHD Metrics}
\renewcommand{\arraystretch}{1.2}
\begin{tabular}{l|ccc}
\toprule
 & \multicolumn{3}{c}{SHD $\downarrow$} \\
Method & $K=0$ & $K=1$ & $K=2$ \\
\midrule
PC & $5.89(4.84)$ & $\underline{7.78}(6.10)$ & $\underline{9.85}(7.02)$ \\
FCI & $9.74(1.97)$ & $11.85(3.61)$ & $12.85(5.31)$ \\
Exact & $20.56(6.17)$ & $21.74(11.16)$ & $24.56(6.32)$ \\
GES & $6.11(4.40)$ & $10.33(7.19)$ & $13.33(20.44)$ \\
GRaSP & $5.96(5.07)$ & $\mathbf{7.22}(6.84)$ & $\mathbf{9.15}(15.16)$ \\
ICA-LiNGAM & $\underline{1.63}(1.94)$ & $12.07(46.59)$ & $16.37(44.60)$ \\
Direct-LiNGAM & $\mathbf{1.48}(1.66)$ & $\mathbf{7.22}(11.58)$ & $12.15(31.31)$ \\
CAM & $16.67(11.04)$ & $20.40(8.31)$ & $23.23(15.98)$ \\
SCORE & $5.07(13.25)$ & $10.67(9.29)$ & $15.67(29.29)$ \\
DAS & $4.22(6.47)$ & $9.07(8.93)$ & $12.13(15.72)$ \\
NoGAM & $5.96(13.52)$ & $11.93(15.06)$ & $15.50(30.85)$ \\
DiffAN & $15.56(19.36)$ & $18.70(26.36)$ & $21.12(21.49)$ \\
CHess (Stein) & $5.93(13.48)$ & $10.40(16.37)$ & $14.40(18.04)$ \\
CHess (DiffAN) & $5.52(9.81)$ & $10.00(12.13)$ & $13.93(18.20)$ \\
\bottomrule
\end{tabular}
\end{table}

\begin{table}[htbp]
\centering
\caption{$d=5$, noise=gumbel: F1 Metrics}
\renewcommand{\arraystretch}{1.2}
\begin{tabular}{l|ccc|ccc}
\toprule
 & \multicolumn{3}{c|}{F1 (Adj) $\uparrow$} & \multicolumn{3}{c}{F1 (Arrow) $\uparrow$} \\
Method & $K=0$ & $K=1$ & $K=2$ & $K=0$ & $K=1$ & $K=2$ \\
\midrule
PC & $\underline{0.98}(0.00)$ & $0.75(0.01)$ & $0.67(0.01)$ & $\backslash$ & $0.41(0.05)$ & $\backslash$ \\
FCI & $\underline{0.98}(0.00)$ & $0.75(0.01)$ & $0.66(0.02)$ & $\backslash$ & $0.41(0.05)$ & $\backslash$ \\
Exact & $0.76(0.01)$ & $0.71(0.01)$ & $0.68(0.01)$ & $0.30(0.03)$ & $0.31(0.03)$ & $0.36(0.03)$ \\
GES & $0.96(0.00)$ & $0.71(0.01)$ & $0.65(0.01)$ & $\backslash$ & $\backslash$ & $\backslash$ \\
GRaSP & $0.95(0.00)$ & $\mathbf{0.76}(0.01)$ & $0.66(0.02)$ & $\backslash$ & $\backslash$ & $\backslash$ \\
ICA-LiNGAM & $\mathbf{1.00}(0.00)$ & $0.72(0.01)$ & $0.67(0.00)$ & $\underline{0.99}(0.00)$ & $0.41(0.05)$ & $0.38(0.04)$ \\
Direct-LiNGAM & $\mathbf{1.00}(0.00)$ & $0.71(0.01)$ & $0.67(0.01)$ & $\mathbf{1.00}(0.00)$ & $0.67(0.01)$ & $0.51(0.03)$ \\
CAM & $0.88(0.01)$ & $0.72(0.01)$ & $0.68(0.01)$ & $0.15(0.02)$ & $0.17(0.03)$ & $0.23(0.02)$ \\
SCORE & $0.90(0.01)$ & $0.73(0.01)$ & $0.69(0.01)$ & $0.62(0.03)$ & $0.41(0.04)$ & $0.38(0.03)$ \\
DAS & $0.90(0.01)$ & $0.73(0.01)$ & $0.69(0.01)$ & $0.62(0.03)$ & $0.41(0.04)$ & $0.38(0.03)$ \\
NoGAM & $0.94(0.01)$ & $0.74(0.01)$ & $\underline{0.70}(0.01)$ & $0.78(0.05)$ & $0.49(0.03)$ & $0.41(0.04)$ \\
DiffAN & $0.87(0.01)$ & $0.75(0.01)$ & $\mathbf{0.70}(0.01)$ & $0.43(0.07)$ & $0.30(0.04)$ & $0.28(0.03)$ \\
CHess (Stein) & $0.97(0.00)$ & $\underline{0.75}(0.01)$ & $\mathbf{0.70}(0.01)$ & $0.94(0.02)$ & $\mathbf{0.69}(0.03)$ & $\mathbf{0.65}(0.02)$ \\
CHess (DiffAN) & $0.95(0.01)$ & $0.75(0.01)$ & $\mathbf{0.70}(0.01)$ & $0.86(0.05)$ & $\underline{0.69}(0.02)$ & $\underline{0.64}(0.01)$ \\
\bottomrule
\end{tabular}
\end{table}

\begin{table}[htbp]
\centering
\caption{$d=5$, noise=gumbel: SHD Metrics}
\renewcommand{\arraystretch}{1.2}
\begin{tabular}{l|ccc}
\toprule
 & \multicolumn{3}{c}{SHD $\downarrow$} \\
Method & $K=0$ & $K=1$ & $K=2$ \\
\midrule
PC & $2.37(1.94)$ & $5.00(3.11)$ & $5.93(1.55)$ \\
FCI & $4.19(1.04)$ & $6.63(1.64)$ & $7.19(1.04)$ \\
Exact & $5.74(1.75)$ & $5.81(2.00)$ & $5.81(1.85)$ \\
GES & $2.59(2.32)$ & $5.89(1.73)$ & $6.33(1.85)$ \\
GRaSP & $2.78(2.10)$ & $4.85(2.50)$ & $6.04(2.85)$ \\
ICA-LiNGAM & $\underline{0.07}(0.07)$ & $5.70(3.39)$ & $6.30(2.65)$ \\
Direct-LiNGAM & $\mathbf{0.00}(0.00)$ & $4.00(1.48)$ & $5.37(1.86)$ \\
CAM & $5.52(2.03)$ & $6.77(2.45)$ & $7.00(2.33)$ \\
SCORE & $2.70(2.13)$ & $5.30(2.54)$ & $6.07(2.60)$ \\
DAS & $2.70(2.13)$ & $5.30(2.54)$ & $6.07(2.60)$ \\
NoGAM & $1.59(3.13)$ & $4.83(2.81)$ & $5.73(2.93)$ \\
DiffAN & $4.15(4.94)$ & $5.78(2.62)$ & $6.56(2.91)$ \\
CHess (Stein) & $0.52(0.99)$ & $\mathbf{3.43}(2.58)$ & $\mathbf{4.10}(1.09)$ \\
CHess (DiffAN) & $1.11(3.14)$ & $\underline{0.35}(2.38)$ & $\mathbf{4.10}(1.22)$ \\
\bottomrule
\end{tabular}
\end{table}

\begin{table}[htbp]
\centering
\caption{$d=10$, noise=gumbel: F1 Metrics}
\renewcommand{\arraystretch}{1.2}
\begin{tabular}{l|ccc|ccc}
\toprule
 & \multicolumn{3}{c|}{F1 (Adj) $\uparrow$} & \multicolumn{3}{c}{F1 (Arrow) $\uparrow$} \\
Method & $K=0$ & $K=1$ & $K=2$ & $K=0$ & $K=1$ & $K=2$ \\
\midrule
PC & $0.96(0.00)$ & $0.77(0.01)$ & $0.66(0.01)$ & $\backslash$ & $0.54(0.02)$ & $0.47(0.02)$ \\
FCI & $0.96(0.00)$ & $\underline{0.77}(0.01)$ & $\underline{0.67}(0.01)$ & $\backslash$ & $0.46(0.01)$ & $0.37(0.01)$ \\
Exact & $0.57(0.00)$ & $0.53(0.00)$ & $0.52(0.00)$ & $0.19(0.01)$ & $0.25(0.01)$ & $0.27(0.02)$ \\
GES & $0.96(0.00)$ & $0.64(0.00)$ & $0.55(0.01)$ & $0.64(0.06)$ & $0.44(0.02)$ & $0.32(0.01)$ \\
GRaSP & $\underline{0.97}(0.00)$ & $\mathbf{0.81}(0.01)$ & $\mathbf{0.69}(0.01)$ & $0.65(0.05)$ & $0.45(0.05)$ & $0.34(0.03)$ \\
ICA-LiNGAM & $0.95(0.01)$ & $0.57(0.01)$ & $0.49(0.00)$ & $\underline{0.82}(0.06)$ & $0.34(0.03)$ & $0.21(0.02)$ \\
Direct-LiNGAM & $\mathbf{0.98}(0.00)$ & $0.60(0.01)$ & $0.52(0.00)$ & $\mathbf{0.98}(0.00)$ & $\mathbf{0.59}(0.01)$ & $\mathbf{0.50}(0.01)$ \\
CAM & $0.86(0.00)$ & $0.70(0.01)$ & $0.61(0.01)$ & $0.20(0.02)$ & $0.21(0.01)$ & $0.17(0.01)$ \\
SCORE & $0.90(0.01)$ & $0.70(0.01)$ & $0.59(0.01)$ & $0.63(0.03)$ & $0.46(0.02)$ & $0.34(0.01)$ \\
DAS & $0.84(0.01)$ & $0.66(0.01)$ & $0.55(0.01)$ & $0.63(0.03)$ & $0.45(0.02)$ & $0.35(0.02)$ \\
NoGAM & $0.90(0.01)$ & $0.69(0.01)$ & $0.60(0.01)$ & $0.73(0.03)$ & $0.50(0.02)$ & $0.41(0.01)$ \\
DiffAN & $0.85(0.01)$ & $0.68(0.01)$ & $0.59(0.01)$ & $0.37(0.03)$ & $0.26(0.02)$ & $0.22(0.01)$ \\
CHess (Stein) & $0.92(0.01)$ & $0.72(0.01)$ & $0.60(0.01)$ & $0.72(0.04)$ & $\underline{0.56}(0.03)$ & $\underline{0.48}(0.02)$ \\
CHess (DiffAN) & $0.90(0.01)$ & $0.70(0.01)$ & $0.60(0.01)$ & $0.66(0.02)$ & $0.52(0.02)$ & $0.40(0.02)$ \\
\bottomrule
\end{tabular}
\end{table}

\begin{table}[htbp]
\centering
\caption{$d=10$, noise=gumbel: SHD Metrics}
\renewcommand{\arraystretch}{1.2}
\begin{tabular}{l|ccc}
\toprule
 & \multicolumn{3}{c}{SHD $\downarrow$} \\
Method & $K=0$ & $K=1$ & $K=2$ \\
\midrule
PC & $4.96(5.29)$ & $\underline{9.37}(6.60)$ & $\mathbf{12.26}(10.64)$ \\
FCI & $9.30(2.43)$ & $13.67(9.56)$ & $16.15(10.13)$ \\
Exact & $20.26(6.19)$ & $21.48(7.95)$ & $20.96(11.89)$ \\
GES & $5.30(9.84)$ & $14.78(9.14)$ & $19.48(20.18)$ \\
GRaSP & $4.81(5.56)$ & $\mathbf{8.96}(7.22)$ & $\underline{12.56}(11.73)$ \\
ICA-LiNGAM & $\underline{2.81}(17.34)$ & $19.11(24.10)$ & $25.11(21.14)$ \\
Direct-LiNGAM & $\mathbf{0.37}(0.46)$ & $13.59(11.28)$ & $18.48(19.81)$ \\
CAM & $10.85(6.64)$ & $14.93(9.40)$ & $18.47(17.78)$ \\
SCORE & $5.52(11.21)$ & $11.47(8.05)$ & $16.53(18.05)$ \\
DAS & $5.67(10.74)$ & $10.67(6.76)$ & $14.80(11.76)$ \\
NoGAM & $4.15(9.24)$ & $11.03(9.77)$ & $15.27(21.40)$ \\
DiffAN & $9.15(10.94)$ & $14.96(15.29)$ & $18.85(17.39)$ \\
CHess (Stein) & $4.00(11.04)$ & $10.07(19.60)$ & $14.17(20.47)$ \\
CHess (DiffAN) & $4.85(6.13)$ & $10.87(13.72)$ & $15.80(24.23)$ \\
\bottomrule
\end{tabular}
\end{table}

\subsection{All Baselines in MLP non-Gaussian ANM}
\paragraph{Results.}  When hard selection bias is not present, the performance of all methods are similar (F1 around $0.7-0.8$). However, when selection bias becomes severe, our conditional Hessian rank-based method starts to outperform while other methods' performance drops significantly. 
\begin{table}[htbp]
\centering
\caption{$d=5$, noise=lognormal, MLP: F1 Metrics}
\renewcommand{\arraystretch}{1.2}
\begin{tabular}{l|ccc|ccc}
\toprule
 & \multicolumn{3}{c|}{F1 (Adj) $\uparrow$} & \multicolumn{3}{c}{F1 (Arrow) $\uparrow$} \\
Method & $K=0$ & $K=1$ & $K=2$ & $K=0$ & $K=1$ & $K=2$ \\
\midrule
PC & $0.79(0.02)$ & $0.67(0.02)$ & $0.64(0.02)$ & $\backslash$ & $\backslash$ & $\backslash$ \\
FCI & $0.79(0.02)$ & $0.67(0.02)$ & $0.64(0.02)$ & $\backslash$ & $\backslash$ & $\backslash$ \\
Exact & $0.62(0.02)$ & $0.62(0.02)$ & $0.65(0.01)$ & $0.33(0.05)$ & $0.32(0.03)$ & $0.34(0.03)$ \\
GES & $0.77(0.01)$ & $0.67(0.03)$ & $0.66(0.02)$ & $\backslash$ & $\backslash$ & $\backslash$ \\
GRaSP & $0.68(0.04)$ & $0.65(0.03)$ & $0.62(0.02)$ & $\backslash$ & $\backslash$ & $\backslash$ \\
ICA-LiNGAM & $0.79(0.01)$ & $0.65(0.01)$ & $0.65(0.01)$ & $0.56(0.04)$ & $0.38(0.03)$ & $0.35(0.03)$ \\
Direct-LiNGAM & $0.75(0.01)$ & $0.67(0.02)$ & $0.67(0.01)$ & $0.50(0.05)$ & $0.40(0.05)$ & $0.45(0.04)$ \\
SCORE & $0.86(0.01)$ & $0.68(0.01)$ & $0.70(0.01)$ & $0.81(0.02)$ & $0.57(0.02)$ & $0.55(0.03)$ \\
DAS & $0.86(0.01)$ & $0.68(0.01)$ & $0.70(0.01)$ & $0.81(0.02)$ & $0.57(0.02)$ & $0.55(0.03)$ \\
CAM & $0.86(0.02)$ & $0.72(0.01)$ & $\mathbf{0.72}(0.01)$ & $0.37(0.06)$ & $0.28(0.03)$ & $0.31(0.04)$ \\
NoGAM & $\mathbf{0.87}(0.01)$ & $0.68(0.02)$ & $0.70(0.01)$ & $\mathbf{0.83}(0.02)$ & $0.57(0.02)$ & $0.55(0.02)$ \\
DiffAN & $0.86(0.01)$ & $\mathbf{0.72}(0.01)$ & $0.71(0.01)$ & $0.65(0.03)$ & $0.47(0.03)$ & $0.36(0.03)$ \\
CHess (Stein) & $0.86(0.02)$ & $0.69(0.01)$ & $\mathbf{0.72}(0.01)$ & $0.77(0.03)$ & $\mathbf{0.64}(0.02)$ & $\mathbf{0.67}(0.02)$ \\
CHess (DiffAN) & $0.85(0.02)$ & $0.70(0.01)$ & $0.71(0.01)$ & $0.78(0.05)$ & $0.62(0.03)$ & $0.65(0.02)$ \\
\bottomrule
\end{tabular}
\end{table}

\begin{table}[htbp]
\centering
\caption{$d=5$, noise=lognormal, MLP: SHD Metrics}
\renewcommand{\arraystretch}{1.2}
\begin{tabular}{l|ccc}
\toprule
 & \multicolumn{3}{c}{SHD $\downarrow$} \\
Method & $K=0$ & $K=1$ & $K=2$ \\
\midrule
PC & $3.77(1.85)$ & $5.67(2.82)$ & $5.83(2.47)$ \\
FCI & $5.70(0.74)$ & $6.73(1.60)$ & $6.93(1.86)$ \\
Exact & $6.23(2.91)$ & $6.53(2.78)$ & $6.33(1.89)$ \\
GES & $4.27(2.46)$ & $5.63(3.37)$ & $6.50(2.98)$ \\
GRaSP & $4.43(2.58)$ & $5.20(2.49)$ & $5.60(1.71)$ \\
ICA-LiNGAM & $3.37(1.77)$ & $6.00(2.20)$ & $6.30(2.48)$ \\
Direct-LiNGAM & $3.87(1.92)$ & $5.70(3.48)$ & $5.67(3.49)$ \\
SCORE & $1.67(1.09)$ & $4.57(1.78)$ & $4.80(2.43)$ \\
DAS & $1.67(1.09)$ & $4.57(1.78)$ & $4.80(2.43)$ \\
CAM & $4.23(4.18)$ & $6.43(2.58)$ & $6.17(2.61)$ \\
NoGAM & $\mathbf{1.47}(1.05)$ & $4.53(2.45)$ & $4.70(2.08)$ \\
DiffAN & $2.67(2.49)$ & $5.07(2.20)$ & $6.00(2.47)$ \\
CHess (Stein) & $1.77(1.65)$ & $\mathbf{4.17}(2.54)$ & $\mathbf{3.90}(1.62)$ \\
CHess (DiffAN) & $1.80(2.96)$ & $\mathbf{4.17}(2.21)$ & $4.03(2.30)$ \\
\bottomrule
\end{tabular}
\end{table}

\newpage
\subsection{DREAM4 experiments}\label{sec:dream4exp}
\paragraph{Results.} We report the SHD of all methods across $5$ networks in Table~\ref{tab:dream4noselect}. Due to the extremely small sample size ($30$ observations for each network), we compare all methods on the full data without selection. We see that our methods achieve the best performance for all networks, with the only exception being Network 5, where DAS~\citep{montagnaScalableCausalDiscovery2023a} achieves the best performance. This shows that our methods perform competitively on small-sized real-world datasets even without selection bias. 
\begin{table}[htbp]
\centering
\caption{SHD Results for $K=0$ Across All Networks}
\label{tab:dream4noselect}
\renewcommand{\arraystretch}{1.2}
\begin{tabular}{l|ccccc}
\toprule
Method & Network 1 & Network 2 & Network 3 & Network 4 & Network 5 \\
\midrule
SCORE & $12.0$ & $12.0$ & $\mathbf{12.0}$ & $13.0$ & $9.0$ \\
DAS & $12.0$ & $13.0$ & $13.0$ & $13.0$ & $\mathbf{8.0}$ \\
NoGAM & $13.0$ & $13.0$ & $13.0$ & $12.0$ & $10.0$ \\
CAM & $12.0$ & $11.0$ & $14.0$ & $12.0$ & $10.0$ \\
CHess (stein) & $\mathbf{11.0}$ & $14.0$ & $13.0$ & $12.0$ & $13.0$ \\
CHess (diffan) & $12.0$ & $\mathbf{9.0}$ & $\mathbf{12.0}$ & $\mathbf{11.8}$ & $10.4$ \\
\bottomrule
\end{tabular}
\end{table}

\subsection{Sachs experiments}\label{sec:sachsexp}

\rev{We provide extra results for the Sachs dataset in Table~\ref{tab:f1_arrow_halfspace}-\ref{tab:f1_arrow_two_boxes} for F1 arrow and Table~\ref{tab:causal_ordering_halfspace}-\ref{tab:causal_ordering_two_boxes} for causal ordering accuracy on the aggretated perturbational dataset and in \ref{tab:f1_arrow_sachs_separate} for F1 arrow on the separated dataset. We vary the strength and shape of the selection bias, and apply random downsampling when $\rho=1$ to the same sample size as when selection bias is present, to enable a fair comparison. The results are consistent with the main text, where our methods achieve the best performance.}
\paragraph{Extra result for \emph{aggregated} Sachs Dataset}
\rev{We combine all 9 experimental conditions into one dataset and score against the ground truth ($n=7466$ before selection bias is applied), which is inline with the main text. We apply selection of different shapes (half-space, annulus, two-box) and different selection fraction $\rho \in \{1.0 \mathrm{(unselected)}, 0.2, 0.1, 0.05\}$.}
\begin{table}[htbp]
\centering
\caption{F1 (Arrow), halfspace selection}
\vspace{3pt}
\label{tab:f1_arrow_halfspace}
\renewcommand{\arraystretch}{1.2}
\setlength{\tabcolsep}{3pt}
\begin{tabular}{l c c c c}
\toprule
Method & \shortstack{$\rho = 1$ \\ $n=1493$ / $747$ / $373$} & \shortstack{$\rho = 0.2$ \\ ($n = 1493$)} & \shortstack{$\rho = 0.1$ \\ ($n = 747$)} & \shortstack{$\rho = 0.05$ \\ ($n = 373$)} \\
\midrule
SCORE & $0.21$ \var{0.04} / $0.26$ \var{0.06} / $0.28$ \var{0.05} & $0.28$ \var{0.04} & $0.29$ \var{0.05} & $0.28$ \var{0.08} \\
DAS & $0.17$ \var{0.06} / $0.18$ \var{0.10} / $0.19$ \var{0.10} & $0.24$ \var{0.05} & $0.23$ \var{0.08} & $0.20$ \var{0.10} \\
NoGAM & $0.23$ \var{0.02} / $0.26$ \var{0.01} / $0.29$ \var{0.04} & $0.28$ \var{0.04} & $0.30$ \var{0.06} & $0.31$ \var{0.08} \\
CAM & $\mathbf{0.41}$ \var{0.06} / $0.42$ \var{0.08} / $0.40$ \var{0.08} & $0.23$ \var{0.08} & $0.28$ \var{0.09} & $0.27$ \var{0.10} \\
PC & $0.27$ \var{0.10} / $0.26$ \var{0.10} / $0.25$ \var{0.09} & $0.23$ \var{0.07} & $0.23$ \var{0.07} & $0.20$ \var{0.09} \\
FCI & $0.16$ \var{0.05} / $0.16$ \var{0.06} / $0.13$ \var{0.04} & $0.11$ \var{0.10} & $0.14$ \var{0.11} & $0.11$ \var{0.13} \\
GES & $0.27$ \var{0.05} / $0.25$ \var{0.05} / $0.25$ \var{0.06} & $0.26$ \var{0.08} & $0.22$ \var{0.08} & $0.23$ \var{0.09} \\
GRaSP & $0.25$ \var{0.07} / $0.24$ \var{0.06} / $0.28$ \var{0.07} & $0.23$ \var{0.11} & $0.22$ \var{0.10} & $0.28$ \var{0.14} \\
ICA-LiNGAM & $0.23$ \var{0.07} / $0.23$ \var{0.06} / $0.27$ \var{0.08} & $0.29$ \var{0.06} & $0.30$ \var{0.07} & $0.27$ \var{0.09} \\
DirectLiNGAM & $0.30$ \var{0.06} / $0.28$ \var{0.06} / $0.30$ \var{0.07} & $0.30$ \var{0.07} & $0.33$ \var{0.09} & $0.31$ \var{0.09} \\
DiffAN & $0.34$ \var{0.06} / $0.36$ \var{0.09} / $0.30$ \var{0.08} & $0.18$ \var{0.08} & $0.19$ \var{0.05} & $0.20$ \var{0.09} \\
CHess (Stein) & $0.40$ \var{0.05} / $\mathbf{0.44}$ \var{0.06} / $\mathbf{0.42}$ \var{0.09} & $0.34$ \var{0.12} & $0.32$ \var{0.12} & $\mathbf{0.36}$ \var{0.08} \\
CHess (DiffAN) & $0.36$ \var{0.06} / $0.42$ \var{0.07} / $0.41$ \var{0.10} & $\mathbf{0.41}$ \var{0.12} & $\mathbf{0.38}$ \var{0.12} & $0.36$ \var{0.10} \\
\bottomrule
\end{tabular}
\end{table}

\begin{table}[htbp]
\centering
\caption{F1 (Arrow), annulus selection}
\vspace{3pt}
\label{tab:f1_arrow_annulus}
\renewcommand{\arraystretch}{1.2}
\setlength{\tabcolsep}{3pt}
\begin{tabular}{l c c c c}
\toprule
Method & \shortstack{$\rho = 1$ \\ $n=1493$ / $747$ / $373$} & \shortstack{$\rho = 0.2$ \\ ($n = 1493$)} & \shortstack{$\rho = 0.1$ \\ ($n = 747$)} & \shortstack{$\rho = 0.05$ \\ ($n = 373$)} \\
\midrule
SCORE & $0.21$ \var{0.04} / $0.26$ \var{0.06} / $0.28$ \var{0.05} & $0.29$ \var{0.06} & $0.32$ \var{0.07} & $0.34$ \var{0.08} \\
DAS & $0.17$ \var{0.06} / $0.18$ \var{0.10} / $0.19$ \var{0.10} & $0.26$ \var{0.06} & $0.29$ \var{0.09} & $0.32$ \var{0.08} \\
NoGAM & $0.23$ \var{0.02} / $0.26$ \var{0.01} / $0.29$ \var{0.04} & $0.28$ \var{0.05} & $0.31$ \var{0.06} & $0.33$ \var{0.09} \\
CAM & $\mathbf{0.41}$ \var{0.06} / $0.42$ \var{0.08} / $0.40$ \var{0.08} & $0.30$ \var{0.08} & $0.31$ \var{0.13} & $0.26$ \var{0.10} \\
PC & $0.27$ \var{0.10} / $0.26$ \var{0.10} / $0.25$ \var{0.09} & $0.23$ \var{0.10} & $0.25$ \var{0.07} & $0.23$ \var{0.12} \\
FCI & $0.16$ \var{0.05} / $0.16$ \var{0.06} / $0.13$ \var{0.04} & $0.11$ \var{0.08} & $0.15$ \var{0.08} & $0.15$ \var{0.11} \\
GES & $0.27$ \var{0.05} / $0.25$ \var{0.05} / $0.25$ \var{0.06} & $0.22$ \var{0.05} & $0.21$ \var{0.07} & $0.20$ \var{0.10} \\
GRaSP & $0.25$ \var{0.07} / $0.24$ \var{0.06} / $0.28$ \var{0.07} & $0.17$ \var{0.07} & $0.18$ \var{0.09} & $0.19$ \var{0.08} \\
ICA-LiNGAM & $0.23$ \var{0.07} / $0.23$ \var{0.06} / $0.27$ \var{0.08} & $0.21$ \var{0.06} & $0.26$ \var{0.06} & $0.27$ \var{0.07} \\
DirectLiNGAM & $0.30$ \var{0.06} / $0.28$ \var{0.06} / $0.30$ \var{0.07} & $0.29$ \var{0.07} & $0.27$ \var{0.07} & $0.30$ \var{0.07} \\
DiffAN & $0.34$ \var{0.06} / $0.36$ \var{0.09} / $0.30$ \var{0.08} & $0.14$ \var{0.05} & $0.20$ \var{0.06} & $0.19$ \var{0.07} \\
CHess (Stein) & $0.40$ \var{0.05} / $\mathbf{0.44}$ \var{0.06} / $\mathbf{0.42}$ \var{0.09} & $\mathbf{0.39}$ \var{0.07} & $\mathbf{0.45}$ \var{0.11} & $\mathbf{0.48}$ \var{0.12} \\
CHess (DiffAN) & $0.36$ \var{0.06} / $0.42$ \var{0.07} / $0.41$ \var{0.10} & $0.36$ \var{0.08} & $0.41$ \var{0.09} & $0.37$ \var{0.14} \\
\bottomrule
\end{tabular}
\end{table}

\begin{table}[htbp]
\centering
\caption{F1 (Arrow), two boxes selection}
\vspace{3pt}
\label{tab:f1_arrow_two_boxes}
\renewcommand{\arraystretch}{1.2}
\setlength{\tabcolsep}{3pt}
\begin{tabular}{l c c c c}
\toprule
Method & \shortstack{$\rho = 1$ \\ $n=1493$ / $747$ / $373$} & \shortstack{$\rho = 0.2$ \\ ($n = 1493$)} & \shortstack{$\rho = 0.1$ \\ ($n = 747$)} & \shortstack{$\rho = 0.05$ \\ ($n = 373$)} \\
\midrule
SCORE & $0.21$ \var{0.04} / $0.26$ \var{0.06} / $0.28$ \var{0.05} & $0.21$ \var{0.06} & $0.24$ \var{0.06} & $0.21$ \var{0.05} \\
DAS & $0.17$ \var{0.06} / $0.18$ \var{0.10} / $0.19$ \var{0.10} & $0.21$ \var{0.05} & $0.21$ \var{0.07} & $0.19$ \var{0.09} \\
NoGAM & $0.23$ \var{0.02} / $0.26$ \var{0.01} / $0.29$ \var{0.04} & $0.26$ \var{0.05} & $0.27$ \var{0.04} & $0.24$ \var{0.04} \\
CAM & $\mathbf{0.41}$ \var{0.06} / $0.42$ \var{0.08} / $0.40$ \var{0.08} & $0.31$ \var{0.09} & $0.34$ \var{0.08} & $0.31$ \var{0.09} \\
PC & $0.27$ \var{0.10} / $0.26$ \var{0.10} / $0.25$ \var{0.09} & $0.29$ \var{0.08} & $0.27$ \var{0.07} & $0.21$ \var{0.12} \\
FCI & $0.16$ \var{0.05} / $0.16$ \var{0.06} / $0.13$ \var{0.04} & $0.13$ \var{0.11} & $0.18$ \var{0.07} & $0.16$ \var{0.11} \\
GES & $0.27$ \var{0.05} / $0.25$ \var{0.05} / $0.25$ \var{0.06} & $0.25$ \var{0.06} & $0.27$ \var{0.09} & $0.28$ \var{0.09} \\
GRaSP & $0.25$ \var{0.07} / $0.24$ \var{0.06} / $0.28$ \var{0.07} & $0.26$ \var{0.09} & $0.27$ \var{0.11} & $0.25$ \var{0.13} \\
ICA-LiNGAM & $0.23$ \var{0.07} / $0.23$ \var{0.06} / $0.27$ \var{0.08} & $0.25$ \var{0.06} & $0.28$ \var{0.06} & $0.26$ \var{0.06} \\
DirectLiNGAM & $0.30$ \var{0.06} / $0.28$ \var{0.06} / $0.30$ \var{0.07} & $0.32$ \var{0.07} & $0.35$ \var{0.08} & $0.32$ \var{0.10} \\
DiffAN & $0.34$ \var{0.06} / $0.36$ \var{0.09} / $0.30$ \var{0.08} & $0.20$ \var{0.08} & $0.23$ \var{0.09} & $0.22$ \var{0.09} \\
CHess (Stein) & $0.40$ \var{0.05} / $\mathbf{0.44}$ \var{0.06} / $\mathbf{0.42}$ \var{0.09} & $0.32$ \var{0.10} & $0.28$ \var{0.11} & $0.29$ \var{0.11} \\
CHess (DiffAN) & $0.36$ \var{0.06} / $0.42$ \var{0.07} / $0.41$ \var{0.10} & $\mathbf{0.37}$ \var{0.10} & $\mathbf{0.40}$ \var{0.14} & $\mathbf{0.39}$ \var{0.12} \\
\bottomrule
\end{tabular}
\end{table}

\begin{table}[htbp]
\centering
\caption{Causal ordering accuracy $1 - D_{top}/|E|$  (0.5 = random order), halfspace selection}
\vspace{3pt}
\label{tab:causal_ordering_halfspace}
\renewcommand{\arraystretch}{1.2}
\setlength{\tabcolsep}{3pt}
\begin{tabular}{l c c c c}
\toprule
Method & \shortstack{$\rho = 1$ \\ $n=1493$ / $747$ / $373$} & \shortstack{$\rho = 0.2$ \\ ($n = 1493$)} & \shortstack{$\rho = 0.1$ \\ ($n = 747$)} & \shortstack{$\rho = 0.05$ \\ ($n = 373$)} \\
\midrule
SCORE & $0.38$ \var{0.06} / $0.42$ \var{0.07} / $0.42$ \var{0.04} & $0.52$ \var{0.08} & $0.52$ \var{0.08} & $0.52$ \var{0.09} \\
DAS & $0.38$ \var{0.06} / $0.42$ \var{0.07} / $0.42$ \var{0.04} & $0.52$ \var{0.08} & $0.52$ \var{0.08} & $0.52$ \var{0.09} \\
NoGAM & $0.42$ \var{0.03} / $0.42$ \var{0.03} / $0.44$ \var{0.04} & $0.53$ \var{0.06} & $0.52$ \var{0.07} & $0.52$ \var{0.06} \\
CAM & $0.64$ \var{0.07} / $0.61$ \var{0.12} / $0.59$ \var{0.12} & $0.40$ \var{0.12} & $0.45$ \var{0.12} & $0.43$ \var{0.14} \\
PC & - / - / - & - & - & - \\
FCI & - / - / - & - & - & - \\
GES & - / - / - & - & - & - \\
GRaSP & - / - / - & - & - & - \\
ICA-LiNGAM & $0.46$ \var{0.12} / $0.44$ \var{0.10} / $0.48$ \var{0.13} & $0.54$ \var{0.10} & $0.56$ \var{0.10} & $0.48$ \var{0.10} \\
DirectLiNGAM & $0.60$ \var{0.11} / $0.58$ \var{0.07} / $0.61$ \var{0.12} & $0.59$ \var{0.12} & $\mathbf{0.62}$ \var{0.14} & $\mathbf{0.57}$ \var{0.13} \\
DiffAN & $0.58$ \var{0.09} / $0.58$ \var{0.11} / $0.56$ \var{0.11} & $0.36$ \var{0.10} & $0.38$ \var{0.07} & $0.41$ \var{0.10} \\
CHess (Stein) & $\mathbf{0.73}$ \var{0.04} / $\mathbf{0.75}$ \var{0.05} / $\mathbf{0.71}$ \var{0.07} & $0.55$ \var{0.13} & $0.48$ \var{0.15} & $0.55$ \var{0.12} \\
CHess (DiffAN) & $0.72$ \var{0.05} / $0.71$ \var{0.06} / $\mathbf{0.71}$ \var{0.09} & $\mathbf{0.66}$ \var{0.12} & $0.60$ \var{0.16} & $0.56$ \var{0.08} \\
\bottomrule
\end{tabular}
\end{table}

\begin{table}[htbp]
\centering
\caption{Causal ordering accuracy $1 - D_{top}/|E|$  (0.5 = random order), annulus selection}
\vspace{3pt}
\label{tab:causal_ordering_annulus}
\renewcommand{\arraystretch}{1.2}
\setlength{\tabcolsep}{3pt}
\begin{tabular}{l c c c c}
\toprule
Method & \shortstack{$\rho = 1$ \\ $n=1493$ / $747$ / $373$} & \shortstack{$\rho = 0.2$ \\ ($n = 1493$)} & \shortstack{$\rho = 0.1$ \\ ($n = 747$)} & \shortstack{$\rho = 0.05$ \\ ($n = 373$)} \\
\midrule
SCORE & $0.38$ \var{0.06} / $0.42$ \var{0.07} / $0.42$ \var{0.04} & $0.45$ \var{0.07} & $0.48$ \var{0.07} & $0.49$ \var{0.06} \\
DAS & $0.38$ \var{0.06} / $0.42$ \var{0.07} / $0.42$ \var{0.04} & $0.45$ \var{0.07} & $0.48$ \var{0.07} & $0.49$ \var{0.06} \\
NoGAM & $0.42$ \var{0.03} / $0.42$ \var{0.03} / $0.44$ \var{0.04} & $0.46$ \var{0.04} & $0.46$ \var{0.05} & $0.47$ \var{0.08} \\
CAM & $0.64$ \var{0.07} / $0.61$ \var{0.12} / $0.59$ \var{0.12} & $0.49$ \var{0.14} & $0.51$ \var{0.17} & $0.45$ \var{0.19} \\
PC & - / - / - & - & - & - \\
FCI & - / - / - & - & - & - \\
GES & - / - / - & - & - & - \\
GRaSP & - / - / - & - & - & - \\
ICA-LiNGAM & $0.46$ \var{0.12} / $0.44$ \var{0.10} / $0.48$ \var{0.13} & $0.42$ \var{0.10} & $0.49$ \var{0.09} & $0.50$ \var{0.10} \\
DirectLiNGAM & $0.60$ \var{0.11} / $0.58$ \var{0.07} / $0.61$ \var{0.12} & $0.58$ \var{0.10} & $0.53$ \var{0.10} & $0.57$ \var{0.11} \\
DiffAN & $0.58$ \var{0.09} / $0.58$ \var{0.11} / $0.56$ \var{0.11} & $0.26$ \var{0.09} & $0.34$ \var{0.13} & $0.34$ \var{0.14} \\
CHess (Stein) & $\mathbf{0.73}$ \var{0.04} / $\mathbf{0.75}$ \var{0.05} / $\mathbf{0.71}$ \var{0.07} & $0.65$ \var{0.16} & $\mathbf{0.71}$ \var{0.09} & $\mathbf{0.71}$ \var{0.10} \\
CHess (DiffAN) & $0.72$ \var{0.05} / $0.71$ \var{0.06} / $\mathbf{0.71}$ \var{0.09} & $\mathbf{0.65}$ \var{0.11} & $\mathbf{0.71}$ \var{0.10} & $0.65$ \var{0.12} \\
\bottomrule
\end{tabular}
\end{table}

\begin{table}[htbp]
\centering
\caption{Causal ordering accuracy $1 - D_{top}/|E|$  (0.5 = random order), two boxes selection}
\vspace{3pt}
\label{tab:causal_ordering_two_boxes}
\renewcommand{\arraystretch}{1.2}
\setlength{\tabcolsep}{3pt}
\begin{tabular}{l c c c c}
\toprule
Method & \shortstack{$\rho = 1$ \\ $n=1493$ / $747$ / $373$} & \shortstack{$\rho = 0.2$ \\ ($n = 1493$)} & \shortstack{$\rho = 0.1$ \\ ($n = 747$)} & \shortstack{$\rho = 0.05$ \\ ($n = 373$)} \\
\midrule
SCORE & $0.38$ \var{0.06} / $0.42$ \var{0.07} / $0.42$ \var{0.04} & $0.39$ \var{0.06} & $0.40$ \var{0.07} & $0.40$ \var{0.07} \\
DAS & $0.38$ \var{0.06} / $0.42$ \var{0.07} / $0.42$ \var{0.04} & $0.39$ \var{0.06} & $0.40$ \var{0.07} & $0.40$ \var{0.07} \\
NoGAM & $0.42$ \var{0.03} / $0.42$ \var{0.03} / $0.44$ \var{0.04} & $0.45$ \var{0.05} & $0.43$ \var{0.06} & $0.41$ \var{0.06} \\
CAM & $0.64$ \var{0.07} / $0.61$ \var{0.12} / $0.59$ \var{0.12} & $0.51$ \var{0.14} & $0.52$ \var{0.15} & $0.51$ \var{0.13} \\
PC & - / - / - & - & - & - \\
FCI & - / - / - & - & - & - \\
GES & - / - / - & - & - & - \\
GRaSP & - / - / - & - & - & - \\
ICA-LiNGAM & $0.46$ \var{0.12} / $0.44$ \var{0.10} / $0.48$ \var{0.13} & $0.48$ \var{0.11} & $0.51$ \var{0.11} & $0.46$ \var{0.10} \\
DirectLiNGAM & $0.60$ \var{0.11} / $0.58$ \var{0.07} / $0.61$ \var{0.12} & $0.65$ \var{0.10} & $0.64$ \var{0.11} & $0.60$ \var{0.14} \\
DiffAN & $0.58$ \var{0.09} / $0.58$ \var{0.11} / $0.56$ \var{0.11} & $0.36$ \var{0.14} & $0.37$ \var{0.12} & $0.43$ \var{0.12} \\
CHess (Stein) & $\mathbf{0.73}$ \var{0.04} / $\mathbf{0.75}$ \var{0.05} / $\mathbf{0.71}$ \var{0.07} & $0.58$ \var{0.13} & $0.55$ \var{0.14} & $0.52$ \var{0.11} \\
CHess (DiffAN) & $0.72$ \var{0.05} / $0.71$ \var{0.06} / $\mathbf{0.71}$ \var{0.09} & $\mathbf{0.67}$ \var{0.09} & $\mathbf{0.66}$ \var{0.12} & $\mathbf{0.61}$ \var{0.13} \\
\bottomrule
\end{tabular}
\end{table}


\paragraph{Extra result for \emph{separate} Sachs Dataset (per-condition)}
\rev{Each of the 9 experimental conditions is treated as its own dataset and scored against its own ground truth: a perfect intervention makes the perturbed protein independent of its parents, so that protein's incoming edges are cut. We apply selection of different shapes (half-space, annulus, two-box) to each of the per-condition subdataset, with $50\%$ of the data remaining. Results show that the proposed CHess achieves the best performance when hard selection is present in the data. Overall CHess (Stein) dominates when selection is present, CHess (DiffAN) dominates in some of the 9 conditions} (\verb+psitect+, \verb+aktinhib+,\verb+pma+) as well, which is not shown in this averaged table.
\begin{table}[htbp]
\centering
\caption{F1 arrow of per-condition Sachs dataset}
\vspace{3pt}
\label{tab:f1_arrow_sachs_separate}
\renewcommand{\arraystretch}{1.2}
\setlength{\tabcolsep}{3pt}
\begin{tabular}{l c c c c}
\toprule
\multirow{2}{*}{Method} & $\rho=1.0$ & \multicolumn{3}{c}{$\rho=0.5$} \\
\cmidrule(lr){3-5}
& non-select & halfspace & annulus & two boxes \\
\midrule
SCORE & $\mathbf{0.37}$ \var{0.09} & $\mathbf{0.36}$ \var{0.07} & $0.37$ \var{0.08} & $0.38$ \var{0.09} \\
DAS & $0.31$ \var{0.08} & $0.31$ \var{0.07} & $0.31$ \var{0.08} & $0.32$ \var{0.07} \\
NoGAM & $0.35$ \var{0.07} & $0.35$ \var{0.07} & $0.35$ \var{0.07} & $0.34$ \var{0.07} \\
CAM & $0.28$ \var{0.11} & $0.28$ \var{0.09} & $0.28$ \var{0.08} & $0.29$ \var{0.11} \\
PC & $0.08$ \var{0.05} & $0.16$ \var{0.02} & $0.11$ \var{0.03} & $0.15$ \var{0.03} \\
FCI & $0.03$ \var{0.05} & $0.09$ \var{0.03} & $0.07$ \var{0.04} & $0.12$ \var{0.06} \\
GES & $0.12$ \var{0.03} & $0.21$ \var{0.02} & $0.18$ \var{0.02} & $0.17$ \var{0.02} \\
GRaSP & $0.07$ \var{0.04} & $0.11$ \var{0.02} & $0.11$ \var{0.05} & $0.11$ \var{0.07} \\
ICA-LiNGAM & $0.34$ \var{0.08} & $0.31$ \var{0.05} & $0.33$ \var{0.06} & $0.34$ \var{0.06} \\
DirectLiNGAM & $0.29$ \var{0.14} & $0.29$ \var{0.07} & $0.29$ \var{0.09} & $0.28$ \var{0.10} \\
DiffAN & $0.27$ \var{0.08} & $0.26$ \var{0.07} & $0.27$ \var{0.08} & $0.27$ \var{0.08} \\
CHess (Stein) & $0.36$ \var{0.13} & $\mathbf{0.36}$ \var{0.11} & $\mathbf{0.37}$ \var{0.13} & $\mathbf{0.39}$ \var{0.12} \\
CHess (DiffAN) & $0.34$ \var{0.06} & $0.34$ \var{0.05} & $0.34$ \var{0.07} & $0.34$ \var{0.06} \\
\bottomrule
\end{tabular}
\end{table}
\section{Hardware}
All experiments were conducted on a server running Debian GNU/Linux 12, equipped with an Intel Xeon E5-2680 v4 CPU (2.40 GHz), 503 GB RAM, and ten NVIDIA GeForce GTX 1080 Ti GPUs (11 GB VRAM each), using Python 3.13.

\newpage

\end{document}